\documentclass[lettersize,journal]{IEEEtran}

\usepackage{amsmath,amssymb,amsfonts}
\usepackage{graphicx}
\usepackage{textcomp}
\usepackage[dvipsnames]{xcolor}
\usepackage{multirow}
\usepackage{booktabs}
\usepackage{siunitx}
\usepackage[ruled,vlined]{algorithm2e}
\usepackage{pgfplots}
\pgfplotsset{width=1\linewidth,compat=1.9}
\pgfplotsset{compat=newest}
\usepackage{tikz}
\usepackage{array}
\usepackage{stfloats}
\usepackage{url}
\usepackage{verbatim}

\usepackage{subcaption}

\usepackage{amsthm}

\newtheorem{theorem}{Theorem}
\newtheorem{lemma}{Lemma}
\newtheorem{assumption}{Assumption}

\newcommand{\our}{FLTP-DR}
\newcommand{\nbo}[1]{{\sf\color{orange}[#1]}}
\newcommand{\ls}[1]{\nbo{LS: #1}}
\newcommand{\yx}[1]{\nbo{YX: #1}}

\newcommand{\blue}[1]{{\sf\color{blue}[#1]}}

\title{Active Client Selection in Federated Trajectory Prediction with Uncertainty-Awareness and Heterogeneous Complexity}

\author{Yiming~Xie, Muzi~Peng, Fei~Miao, Ningfang~Mi, and Lili~Su%
\thanks{Y. Xie, M. Peng, N. Mi, and L. Su are with Northeastern University, Boston, MA, USA (e-mail: xie.yimi@northeastern.edu; peng.mu@northeastern.edu; ningfang@ece.neu.edu; l.su@northeastern.edu).}%
\thanks{F. Miao is with the University of Connecticut, Storrs, CT, USA (e-mail: fei.miao@uconn.edu).}%
}

\begin{document}

\maketitle

\begin{abstract}
Training sequence models (e.g., transformers) is now standard for autonomous vehicle trajectory prediction. Yet, assembling high-quality centralized datasets is challenging because real-world vehicles' trajectories are fragmented across regions and vehicles. Federated Learning (FL) offers a natural alternative. However, its application to trajectory prediction presents two distinctive challenges: (i) \emph{high scene uncertainty}, which arises from trajectories or map ambiguity in individual traffic scenes, 
and (ii) \emph{cross-scene complexity heterogeneity}, which is driven by the diversity of map topology, traffic density, agent composition, and driving behaviors.  
By incrementally establishing awareness of scene uncertainty and complexity heterogeneity, we propose 
a family of \emph{active client selection} methods for FL that guide client selection toward the most informative ones. 
Our uncertainty-aware selectors prioritize clients using (1) per-client 
negative log-likelihood (NLL) under an uncertainty-aware global objective and (2) estimated aleatoric uncertainty.
Inspired by the intuition that more complex scenes often contain knowledge that can be transferred down to easier scenes, we further develop a selector (named {\em \our}) that  
jointly accounts for scene complexity and uncertainty.

Experiments on Argoverse demonstrate that even vanilla federated trajectory prediction surpasses locally trained models. Uncertainty-aware selection accelerates convergence and improves core metrics, including minADE, minFDE, and MR. When clients' local dataset exhibits significant heterogeneity in scene complexities, 
{\our} attains the best generalization while further speeding convergence, confirming 
the conjecture 
that training on more complex scenes benefits performance on easier 
cases. As a byproduct, we provide a reproducible setup for constructing highly heterogeneous federated partitions of trajectory datasets, which may be of independent interest.

\end{abstract}



\section{Introduction}

\IEEEPARstart{A}{ccurate} trajectory prediction of surrounding objects is crucial for autonomous driving. Recently, 
complex deep networks, such as transformers,  
have become the method of choice for trajectory prediction \cite{bansal2018chauffeurnet,cui2019multimodal,djuric2020uncertainty,liang2020learning,gao2020vectornet,zhou2022hivt}. Due to the data-intensive nature of deep learning, existing methods implicitly assume the availability of large-scale, centralized datasets \cite{zhou2022hivt,liang2020learning,ngiam2021scene,gu2021densetnt,liu2021multimodal,weng2022whose,zhang2022adversarial,bahari2022vehicle}. However, real-world data is often fragmented across devices.  
In this paper, we leverage Federated Learning (FL) on Connected Autonomous Vehicles (CAVs) to relax the reliance on centralized data for high-precision model training. 
FL is a promising decentralized learning paradigm \cite{mcmahan2017communication,kairouz2021advances} that allows CAVs to collaboratively train a global model without exchanging raw data. 
Under FL, each vehicle locally trains the model with its local dataset and periodically sends updates (e.g., model weights) to a central parameter server, which aggregates them into an enhanced global model and redistributes it for continued local training. This approach effectively leverages diverse local data to improve generalization. 

Trajectory prediction in FL faces two distinctive challenges: (i) \emph{high scene uncertainty} and (ii) \emph{cross-scene complexity heterogeneity}, illustrated in Figure~\ref{fig:uncertainty_heterogeneity}. 
The left subfigure of Figure~\ref{fig:uncertainty_heterogeneity} showcases \emph{high scene uncertainty}. At a four-way intersection, the ego vehicle may execute multiple valid driving maneuvers, such as going straight, turning left, and turning right. Based on the past trajectory, it is ambiguous for the ego vehicle to predict the exact trajectory. 
The three curves depict these plausible future trajectories. 
The small orange rectangle marks a potentially out-of-distribution (OOD) zone (e.g.\, partially blocked construction zone) where ambiguity increases. The right subfigure of Figure~\ref{fig:uncertainty_heterogeneity} illustrates the \emph{cross-scene complexity heterogeneity}. Urban, suburban, and highway clients possess different maps and traffic statistics (intersection proximity, traffic density, turn frequencies, and lane merges). 
In particular, metropolitan traffic scenes are generally more complex than rural scenes, characterized by denser traffic, more frequent intersections, and richer pedestrian interactions, whereas rural scenes typically involve lighter traffic and fewer interactions. Furthermore, regional differences (e.g., snowy regions versus warmer climates) give rise to diverse driving behaviors. 
%
%
\begin{figure}[!t]
    \centering
    \includegraphics[width=0.9\columnwidth]{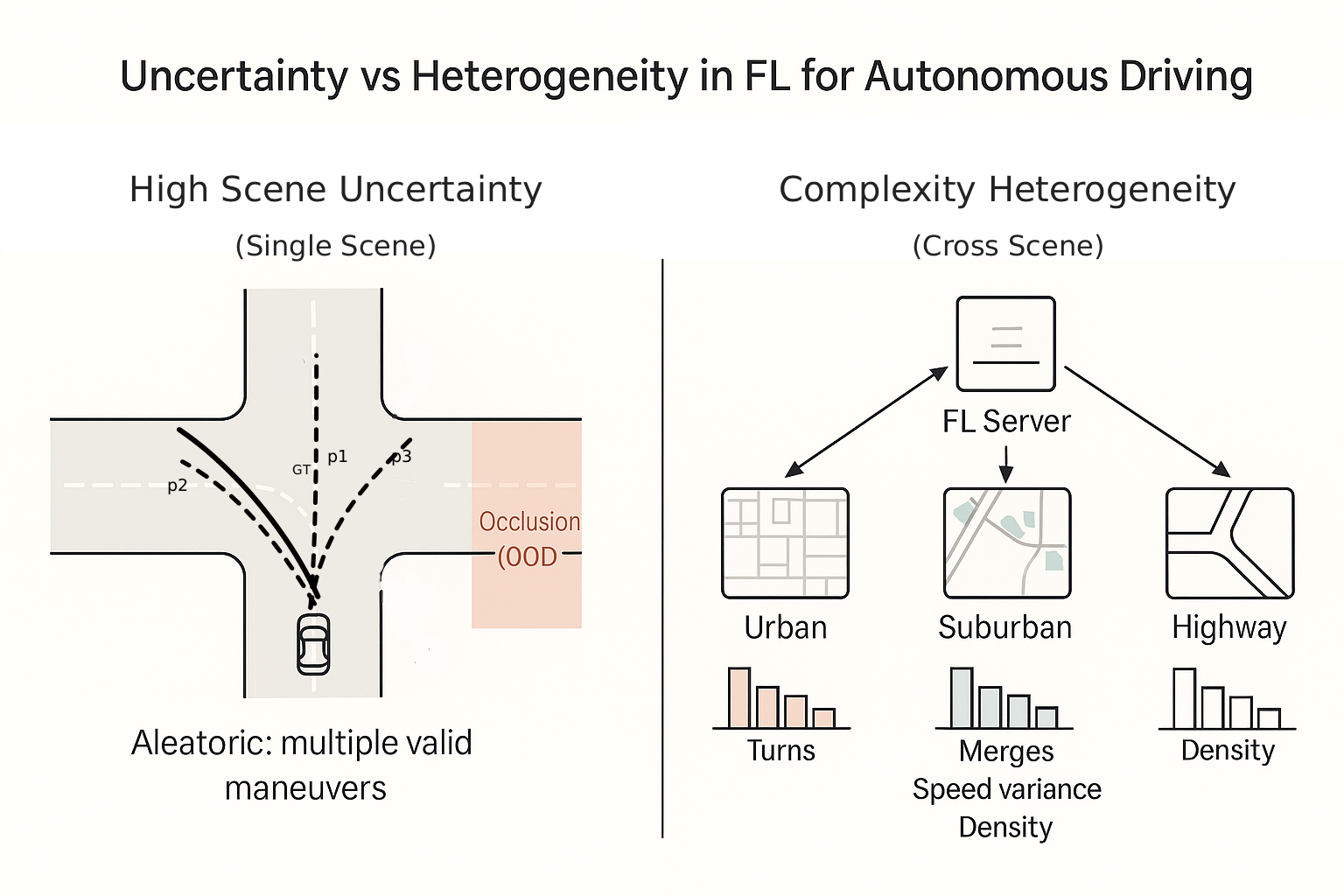}
    \caption{Uncertainty and heterogeneity in FL for autonomous driving. 
    Left: uncertainty in individual traffic scenes---the past trajectory of a vehicle may admit multiple valid futures and exhibit higher uncertainty in OOD scenes. 
    Right: scene complexity heterogeneity--arising from different map and traffic pattern distributions (urban, suburban, highway).}
    \label{fig:uncertainty_heterogeneity}
\end{figure}


Traditional FL algorithms randomly sample clients for aggregation \cite{mcmahan2017communication, li2020federated, bonawitz2019towards}. 
Despite resilience to non-IID data having been studied intensively in the FL literature \cite{kairouz2021advances}, the scene complexity heterogeneity we encounter necessitates a fundamental rethinking of existing methods. This is because, due to attention mechanisms and intricate agent–agent interactions in traffic scenes, weight training 
of many transformer-based trajectory prediction models vary with scene complexity -- particularly traffic density \cite{zhou2022hivt,hpnet,wangtowards,yang2024sstp}. As a result, their theoretical guarantees against data heterogeneity break down, since the widely adopted {\em bounded gradient dissimilarity} assumption \cite{mcmahan2017communication,li2020federated,karimireddy2020scaffold,stich2019local,woodworth2020local,wang2022unreasonable} is no longer feasible.
%
%
These considerations motivate active client selection that prioritizes clients based on informativeness criteria tied to \emph{uncertainty} and \emph{complexity heterogeneity}.

\noindent\textbf{Contributions.}
To address these gaps, we introduce a comprehensive FL framework tailored to trajectory prediction with explicit treatment of uncertainty and cross-client complexity heterogeneity. Our contributions are:

\begin{enumerate}
\item  \textbf{FLTP: Federated Trajectory Prediction with an uncertainty-aware objective.}
 We present \emph{FLTP}, which trains trajectory forecasters in a decentralized manner without sharing raw data.  
FLTP adopts a negative log-likelihood (NLL) objective (e.g., a Laplace-mixture head) to model multi-modality and quantify aleatoric uncertainty (AU), paving the way for uncertainty-awareness. 


\item \textbf{Active client selection under partial client participation.}
We provide a family of active client selection methods 
that incrementally incorporates uncertainty-awareness and scene complexity heterogeneity.  
For uncertainty, we develop FLTP-AU and FLTP-NLL. FLTP-AU prioritizes clients with medium \emph{aleatoric uncertainty}, treating AU as the primary signal of scene ambiguity. \textbf{FLTP-NLL} selects clients with larger per-client negative log-likelihood under a Laplace–mixture objective, indicating poorer fit and elevated ambiguity. 
To complement uncertainty-awareness, {\our} incorporates heterogeneity-awareness by combining a complexity rank constructed from interpretable features (turn frequency, intersection proximity, lane merges, speed variance, and interaction density) with local training loss to emphasize challenging, learnable data. 

\item \textbf{Complexity-aware scoring \& realistic federated partitioning.} 
We introduce an interpretable \emph{trajectory complexity} score capturing scene complexity that includes factors such as lane topology and agent-agent/map interactions. 
Using this score, we (i) rank traffic scenes 
as Easy/Moderate/Hard and (ii) construct realistic non-IID client splits with a Dirichlet partitioner (tunable concentration \(\alpha\)~\cite{Yurochkin2019BayesianNF} 
and mix parameter \(\nu\), covering the spectrum from highly imbalanced to nearly uniform deployments.

\item \textbf{Knowledge transfer across complexity levels.}
We demonstrate that training on more complex traffic scenes expands maneuver or topology coverage and transfers to moderate or easy scenes: Overall, increasing the fraction of Hard data (\(\nu\)) 
improves minADE and minFDE while reducing MR on a fixed validation distribution. This validates our central conjecture that the proposed complexity score modeling helps the system handle a wide range of trajectory dynamics rather than overfitting to a specific -- though common --  group of traffic scenes. 
\end{enumerate}

\section{Related Work}

\subsection{Trajectory Prediction for Autonomous Vehicles}
 
The ability of predicting future trajectories of surrounding agents is important for the safety of autonomous driving. 
The pipeline of trajectory prediction consists of three subtasks: input representation, context aggregation, and output representation. 

Input representation is typically generated using either rasterization \cite{bansal2018chauffeurnet, cui2019multimodal, djuric2020uncertainty} or vectorization \cite{liang2020learning, gao2020vectornet, zhou2022hivt}. While rasterized representations are well-suited for mature CNN backbones, they suffer from information loss and high computational costs. 
Recent research has increasingly focused on vectorized representations, where elements of HD maps and agents’ past trajectories are encoded as nodes within a graph structure. Compared to rasterized representations, vectorized representations mitigate information loss and enable more efficient encoding of complex topologies and long-term agent dynamics. 
To capture critical traffic interactions—such as vehicle-to-vehicle, vehicle-to-lane, and vehicle-to-pedestrian interactions—context aggregation modules employ techniques like social pooling \cite{deo2018convolutional, bansal2018chauffeurnet}, attention mechanisms \cite{ngiam2021scene, zhou2022hivt, gu2021densetnt, liu2021multimodal}, and Graph Neural Networks (GNNs) \cite{liang2020learning, djuric2020uncertainty, weng2022whose}. These methods extract and fuse contextual information from agent dynamics, road maps, and interactions among traffic participants, improving the accuracy of trajectory forecasting.
Finally, a trajectory prediction model often generates predictions of future trajectories using either regression-based \cite{deo2018convolutional, liang2020learning, zhou2022hivt} or proposal-based approaches \cite{chai2019multipath, phan2020covernet, zhao2021tnt, liu2021multimodal}. Regression-based methods directly predict multiple deterministic future trajectories from the extracted context features. In contrast, proposal-based methods first generate candidate trajectories or anchor points based on prior knowledge, making them inapplicable to our setting due to a lack of such prior knowledge.  

However, state-of-the-art models rely heavily on training with large-scale datasets, which may not be available in practice due to privacy concerns. 
In this work, we focus on distributed training two representative transformers - HiVT \cite{zhou2022hivt} and HPNet \cite{hpnet} through Federated Learning.  
Both models represent input through vectorization, capture complex and dynamic multi-agent interactions, and generate predictions without relying on prior knowledge \cite{zhou2022hivt,hpnet}.  
In addition, they are open-sourced and achieved the state-of-the-art at the time of their introduction.  HiVT is one of the earlier models that can efficiently encode the complex interactions of a large number of agents in challenging driving scenes.  
HPNet is a newer model; unlike conventional models, it enforces consistency constraints between past and present predictions, achieving the state-of-the-art at the time this paper was prepared.


\subsection{Federated Learning for Trajectory Prediction}
%

Despite the broad applicability of FL in various machine learning contexts, its adoption for trajectory prediction tasks—particularly in autonomous driving—is still relatively limited.
Trajectory prediction for autonomous vehicles (AVs) presents unique challenges, including real-time prediction requirements, complex spatial-temporal interactions, and heterogeneous data distributions caused by varying driving behaviors, environments, and sensor capabilities.

Recent studies have explored FL for trajectory prediction in related settings. Flow-FL~\cite{majcherczyk2021flowfl} applies FL to trajectory prediction for connected robot teams, demonstrating the feasibility of decentralized learning in robotic environments.
However, it does not address the complexity and variability of autonomous driving scenarios.
%
ATPFL~\cite{wang2022atpfl} combines automated machine learning (AutoML) with federated learning for human trajectory prediction.
While effective for pedestrian prediction, ATPFL does not consider complex vehicle interactions or adaptive strategies for heterogeneous data, which are critical in AV applications.

Several studies have also applied FL to autonomous driving perception tasks. For example, Liu et al.~\cite{liu2021federated} proposed a federated framework for V2X perception, while FedVCP~\cite{zhang2022federated} enables collaborative visual context perception among vehicles. Although these works demonstrate the potential of FL in connected vehicle systems, they focus on perception rather than trajectory forecasting and do not address adaptive client selection.

To the best of our knowledge, federated learning for real-world connected and autonomous vehicle (CAV) trajectory prediction remains largely unexplored. Existing studies do not comprehensively address data heterogeneity, environmental variability, and uncertainty in trajectory forecasting. Our work fills this gap by introducing an adaptive federated learning framework tailored to CAV trajectory prediction, incorporating adaptive client selection and uncertainty-aware training to improve robustness under heterogeneous data distributions.

\subsection{Active Client Selection in Federated Learning}



The idea behind active learning is to identify data samples that are more informative for model training. Several existing approaches aim to improve FL performance through biased client selection. Some methods focus on accelerating the convergence rate by selecting clients with faster computation capabilities or larger datasets \cite{nishio2019client, ribero2020communication}. Others prioritize higher average accuracy by choosing clients based on local model performance or gradient updates \cite{goetz2019active, cho2020client}. Furthermore, fairness-aware FL strategies, such as $q$-Fair FL, aim to balance participation among clients to mitigate biases caused by data heterogeneity \cite{li2020fair, huang2020efficiency}.

Inspired by active learning, several works propose active client selection strategies to enhance model training efficiency. Some research papers \cite{goetz2019active, cho2020client, xie2024adaptive} leverage local training loss as an active learning metric, selecting clients with higher loss values to facilitate model improvement. Similarly, Li et al. \cite{li2022uncertainty} introduce Bayesian active learning into FL, where clients are chosen based on model uncertainty. These methods demonstrate the potential of selective participation in FL, improving both convergence and model generalization. However, existing strategies predominantly focus on epistemic uncertainty or loss-based metrics, with limited exploration of data complexity itself as a client selection criterion.

Our work addresses this gap by integrating both data complexity (high rank clients) and training loss-based selection to improve FL performance. Unlike prior approaches that consider only a single metric for client selection, our method adaptively balances high-rank and high-loss client selection, optimizing both generalization and robustness in federated settings.
\section{Problem Formulation}
\label{sec: problem formulation}

In this section, we formally describe the federated trajectory prediction problem for connected autonomous vehicles. We introduce the data setup,
and explain the roles of each client and the central server within our federated learning framework. 
We also present, at a high level, two trajectory prediction models, HiVT \cite{zhou2022hivt} and HPNet \cite{hpnet}, that will be used in our methods; details are deferred to supplementary materials for completeness. 
\subsection{Setup}
A driving scene data (i.e.\,a data sample) 
can be described by a triple $S=(X, Y, \mathcal{M})$, where $X$ and $Y$ are the collections of observed and future trajectories of the involved agents (an agent can be a vehicle or a pedestrian), and $\mathcal{M}$ is the map.  
Let $m$ denote the number of agents in the scene, then $X$ and $Y$ can be expressed as $X = \{x_1,...,x_m$\} and $Y = \{y_1,...,y_m$\}, 
where $x_i \in \mathbb{R}^{2 \times T_{obs}}$ and $y_i \in \mathbb{R}^{2 \times T_{pre}}$ are the two-dimensional observed and future trajectories of agent $i$, with lengths $T_{obs}$ and $T_{pre}$, respectively. In particular, for any given scene $S$, we designate a target agent, denoted by $i^*$, to evaluate the performance of the trajectory prediction model. 

The system contains a server and $C$ autonomous vehicles, each of which can collect its driving scene data using Lidar sensors and cameras to record the trajectories of all neighboring vehicles.   
We refer to each autonomous vehicle as one ego vehicle, which serves as one client in our FL framework. 
Each client $c\in \mathcal{C} \triangleq \{1, \cdots, C\}$ has a local dataset of size $K_c$, denoted as $\mathcal{D}_c = \{S^1, S^2,...,S^{K_c}\}$. Let $m_c^k$ be the number of agents in the $k$-th sample of client $c$. Let $K = \sum\limits_{c=1}^{C}K_c$ denote the total number of samples. 

\subsection{Backbones at a Glance: HiVT and HPNet}

\textbf{HiVT} (Hierarchical Vector Transformer)~\cite{zhou2022hivt} is a lightweight transformer forecaster that separates local context extraction from global interaction modeling for multi-agent, multi-modal prediction. In FL, we train HiVT with an uncertainty-aware likelihood objective.

\textbf{HPNet} (Hierarchical Prediction Network)~\cite{hpnet} is a more recent backbone with a stronger inductive bias for long-horizon, structured motion. Beyond the generic “coarse-to-fine” notion, HPNet introduces: (i) a \emph{hierarchical decoder} that first produces a global, low-frequency future layout (stabilizing long-range intent) and then performs local refinements (capturing short-range maneuver details); (ii) \emph{multi-scale temporal fusion} that aggregates context across different temporal resolutions, improving recall of long-range cues while preserving near-term dynamics; and (iii) an \emph{interaction-aware} fusion stage that better conditions local refinements on surrounding agents’ evolution. In practice, these design choices lead to stronger modeling capacity on complex scenes (turns, merges, near-intersection behavior) and more stable gradients under heterogeneous client data. For this reason, we adopt HPNet as the backbone when evaluating our heterogeneity-aware selector ({\our}), so the selection logic is not bottlenecked by a weaker decoder.

\section{Methodology}
\label{sec: method}
%


In this methodology section, we first introduce a generic FL algorithm for collaborative trajectory prediction model training, paving the way for active client selection. 
Then, we develop a family of active client selection methods that incrementally build awareness of uncertainty and complexity heterogeneity. 


\subsection{Federated Learning Based Trajectory Prediction (FLTP)}
\label{subsec: FLTP} 
%
In FL, as the client can only get access to its local data, the local objective of client $c$ is defined as:
\begin{equation}
\label{eq: local objective}
F_c(\mathcal{D}_c,w)
=
\frac{1}{K_c}
\sum_{S\in\mathcal{D}_c}
L(S,w),
\end{equation}

where $L(\cdot, \cdot)$ is a loss function chosen according to the model of interest. 
In trajectory prediction, it typically combines a trajectory fit term with a probability model for multi-modal futures.  
For example, for HiVT, $L$ combines the regression loss based on the NLL of the Laplace distribution with cross entropy loss (for mode classification), whereas in HpNet, $L$ is primarily the mean squared error (MSE) loss.
 
We formally describe our FLTP in Algorithm \ref{alg:fltp_server}. 
It follows the general server-client interaction of FedAvg \cite{mcmahan2017communication}. Departing from the standard FedAvg, instead of stochastic gradient descent, we use AdamW as the local optimizer.  
%
%
We adopt AdamW because its adaptive preconditioning and decoupled weight decay improve stability under small local batches and non-IID drifts, reduce sensitivity to learning-rate tuning, and enable faster and more robust convergence across heterogeneous clients.
In addition, before proceeding to model update iterations, each client computes and reports the complexity scores of its local datasets, denoted by $\bar{\kappa}_c$, to the parameter server. 
The definition and computation of $\bar{\kappa}_c$ are deferred to Section \ref{subsec: complexity selector}. 
In each global iteration of Algorithm \ref{alg:fltp_server}:\footnote{It is worth noting that the algorithm can be improved by using a global stepsize. We leave this direction to future work. } 
\begin{itemize}
\item 
The parameter server first forms a \emph{candidate pool} \(\mathcal{C}_r \subset \mathcal{C}\) according to the given pooling criteria.\footnote{Popular examples of the pooling criteria include full participation, uniformly at random, and random selection with/without replacement.}
\item Then the parameter server sends the current model $w_{r-1}$ to each of the candidate clients in $\mathcal{C}_r$ 
to update models on their local data. 
\item In parallel, each of the pooled clients in $\mathcal{C}_r$ runs AdamW with respect to Eq.\eqref{eq: local objective} with the specified minibatch size $B$ for $E$ epochs on local dataset. Concretely, in the \textsf{ClientUpdate} function, $\theta$ and $\Sigma$ are the weighted cumulative first and second moments, respectively, of the mini-batch gradients observed so far with $\beta_1, \beta_2 \in (0, 1)$ as the momentum parameters. Depending on the minibatch size $B$, for the first few iterations in the inner {\bf for}-loop, the smallest eigenvalues of $\hat{\Sigma}$ could be either zero or extremely small -- resulting in significant fluctuation of $w$. Hence, $\epsilon>0$ is used to smooth the updates. 
\item The reported local updates $\{w_r^c\}_{c\in \mathcal{C}_r}$, $\mathcal{C}_r$, and the complexity scores $\{\bar{\kappa}_c\}_{c\in \mathcal{C}}$ are then passed to a client selection function to output a subset $\mathcal{L}_r\subseteq \mathcal{C}_r$, which we will further develop in Section \ref{subsec: active learning} and \ref{subsec: complexity selector}.   
\item Finally, the parameter server aggregates the updates $w_{r}^c$ reported by clients in $\mathcal{L}_r$ to obtain $w_{r}$. 
\end{itemize}


\begin{algorithm}[ht]
\caption{\small FLTP with ClientUpdate}
\label{alg:fltp_server}
\small 
\KwIn{Initial model $w_0$, number of clients $C$, local data volumes $\{K_1, \cdots, K_C\}$, stepsize $\eta$, batch size $B$, number of epochs $E$, weight decay $\lambda$, momentum parameters $\beta_1, \beta_2 \in (0,1)$, smoothing parameter $\epsilon$}
\KwOut{Final model $w_R$}

Initialize global model: $w \leftarrow w_0$\;
\For{each client $c\in \mathcal{C}$}
{Computes and reports $\bar{\kappa}_c$ -- the complexity score of its local dataset. }

\For{each round $r = 1$ to $R$}{
    Select candidate pool $\mathcal{C}_r$\;
    Send $w_{r-1}$ to all clients in $\mathcal{C}_r$\;
    \For{each client $c \in \mathcal{C}_r$ in parallel}{
        $w_r^c \leftarrow$ \texttt{ClientUpdate}$(w_{r-1}, \eta, B, E, \beta_1, \beta_2, \epsilon, \lambda)$\;
        Compute local selection value $v_c$; 
    }
    $\mathcal{L}_r \leftarrow \text{\texttt{SelectFunction}}(\mathcal{C}_r, \{v_c\}_{c\in \mathcal{C}_r}, \{\bar{\kappa}_c\}_{c\in \mathcal{C}}, k, \text{mode})$\;
    Aggregate updates: $w_r \leftarrow \sum_{c \in \mathcal{L}_r} \frac{K_c}{\tilde{K}_r} w_r^c$, where $\tilde{K}_r = \sum_{c \in \mathcal{L}_r} K_c$\;

}
\Return $w_R$\;

\textbf{Procedure} \texttt{ClientUpdate}($w, \eta, B, E, \beta_1, \beta_2, \epsilon, \lambda$):

Initialize: $\theta \leftarrow 0$, $\Sigma \leftarrow 0$, $t \leftarrow 0$\;
\For{each local epoch $t = 1$ to $E$}{
    Divide local dataset $\mathcal{D}_c$ into batches $\mathcal{B}_c$ of size $B$\;
    \For{each batch $b \in \mathcal{B}_c$}{
        $t \leftarrow t + 1$\;
        Compute gradient: $g = \nabla F_c(b, w)$\;
        Apply weight decay: $w \leftarrow w - \eta \lambda g$\;
        Update moments: $\theta \leftarrow \beta_1 \theta + (1 - \beta_1) g$, $\Sigma \leftarrow \beta_2 \Sigma + (1 - \beta_2) gg^\top$\;
        Bias correction: $\hat{\theta} \leftarrow \theta / (1 - \beta_1^t)$, $\hat{\Sigma} \leftarrow \Sigma / (1 - \beta_2^t)$\;
        Update model: $w \leftarrow w - \eta (\hat{\Sigma}^{1/2} + \epsilon I)^{-1} \hat{\theta}$\;
    }
}
\Return $w$\;
\end{algorithm}

\begin{algorithm}[!t]
\caption{\small \texttt{SelectFunction}: Client-selection subroutine}
\label{alg:select_function}
\small 
\KwIn{Candidate pool $\mathcal{C}_r$, $k$, $w_{r}$, mode $\in\{\textsf{uniform},\textsf{NLL},\textsf{AU}, \textsf{DR}\}$,  $\{v_c\}_{c\in\mathcal{C}_r}$}
\KwOut{Selected set $\mathcal{L}_r$ or $\mathcal{L}_r^{(\kappa)}$}

\Switch{\text{mode}}{%
  \Case{\textsf{uniform}}{%
    Uniformly sample $k$ clients \emph{without replacement} from $\mathcal{C}_r$ to form $\mathcal{L}_r$\;
  }%
  \Case{\textsf{NLL}}{%
    Sort $\{v_c\}_{c\in\mathcal{C}_r}$ in descending order; take the top-$k$ to form $\mathcal{L}_r$\;
  }%
  \Case{\textsf{AU}}{%
    Compute $m \gets \operatorname{median}\big(\{v_c\}_{c\in\mathcal{C}_r}\big)$; select the $k$ clients with smallest $|v_c - m|$ to form $\mathcal{L}_r$\;
  }%
  \Case{\textsf{DR}}{%
    Sort $\{\bar{\kappa}_c\}_{c\in\mathcal{C}_r}$ in descending order; take the top-$k$ to form $\mathcal{L}_r^{(\kappa)}$\;
  }%
}
\Return $\mathcal{L}_r$\;
\end{algorithm}

\subsection{Uncertainty-aware Client Selectors} 
\label{subsec: active learning}
Departing from existing literature \cite{cho2020client}, instead of directly using the whole loss function as the metric, we consider two types of uncertainty-aware client selection metrics: NLL and AU. 

In our Algorithm \ref{alg:alfltp}, each client has a value variable $v_c$. In each round $r$ (where $r\ge2$), the parameter server samples the clients twice.  First, it randomly chooses $\lfloor f_2C \rfloor$ clients as a pool of candidates $\mathcal{C}_r$.  
Each of the chosen clients in $\mathcal{C}_r$ updates its local selection value $v_c$ as 
$$
v_c =  G_c(\mathcal{D}_c,w) \triangleq  \frac{1}{K_c}\sum\limits_{S\in\mathcal{D}_c}G(S,w),
$$
where 

\begin{equation}
\resizebox{\columnwidth}{!}{$
G(S,w)=
\begin{cases}
\frac{1}{T_{\mathrm{pre}}}\sum_{t=1}^{T_{\mathrm{pre}}}
\left[\log\!\Bigl(2\hat{b}_{i^\star,t,f_{\text{best}_{i^\star}}}\Bigr)+
\frac{\lVert y_{t}-\hat{\mu}_{i^\star,t,f_{\text{best}_{i^\star}}}\rVert_1}
{\hat{b}_{i^\star,t,f_{\text{best}_{i^\star}}}}\right], & \text{NLL},\\[4pt]
\frac{1}{T_{\mathrm{pre}}}\sum_{t=1}^{T_{\mathrm{pre}}}
\hat{b}_{i^\star,t,f_{\text{best}_{i^\star}}}, & \text{AU},
\end{cases}
$}
\label{eq:G}
\end{equation}
where $\hat{\mu}_{i^*,t,f_{best_{i^*}}}$ denotes the prediction location of the target agent at time frame $t$ in scene $S$ of the best mode of $F$ trajectories and $\hat{b}_{i^*,t,f_{best_{i^*}}}$ denotes the corresponding aleatoric uncertainty.
All clients that are not contained in $\mathcal{C}_r$ reset their selection values as $v_c=0$.  
When NLL is used as the selection metric, then the parameter server chooses the set $\mathcal{L}_r$ to be the $\lfloor f_1C \rfloor$ clients with the highest $v_c$.
When AU is used as the metric, then the parameter server chooses $\mathcal{L}_r$ to contain the $\lfloor f_1C \rfloor$ whose values $v_c$ are closest to their median.  

When $r=1$, 
as there is no global model from the last round for value calculation, the client selection method is similar to that in FLTP.

\begin{algorithm}[t]
\caption{\small FLTP-NLL/AU}
\label{alg:alfltp}
\small 
\KwIn{Initial model $w_0$, stepsize $\eta$, epochs $E$, client sampling rate $f_1\!\in\!(0,1]$, candidate sampling rate $f_2\!\in\!(0,1]$, $\{K_1,\dots,K_C\}$, weight decay $\lambda$, momentum $\beta_1,\beta_2\!\in\!(0,1)$, smoothing $\epsilon$, selection mode $u\in\{\textsf{NLL},\textsf{AU}\}$}

\KwOut{Final model $w_R$}

Initialize $w \leftarrow w_0$\;
Uniformly-at-random sample $\lfloor f_2 C\rfloor$ clients \emph{without replacement} to form $\mathcal{C}_1$\; 
Broadcast $w_{0}$ to all $c\in\mathcal{C}_1$\;
\ForEach{$c \in \mathcal{C}_r$ in parallel}{
    $w_1^c \leftarrow \texttt{ClientUpdate}(w_0, \eta, B, E, \beta_1, \beta_2, \epsilon, \lambda)$\;
    $v_c \gets 0$\; 
}
$\mathcal{L}_1 \leftarrow \texttt{SelectFunction}(\mathcal{C}_r, \{v_c\}_{c\in \mathcal{C}_r}, \lfloor f_2 C\rfloor, \textsf{uniform})$\;
Aggregate: $w_1 \leftarrow \sum_{c \in \mathcal{L}_1} \frac{K_c}{\tilde{K}_1}\, w_1^c$, where $\tilde{K}_1 = \sum_{c \in \mathcal{L}_1} K_c$\;

\For{$r = 2$ to $R$}{
    Uniformly-at-random sample $\lfloor f_2 C\rfloor$ clients without replacement to form candidate pool $\mathcal{C}_r$\;  
    Broadcast $w_{r-1}$ to all $c\in\mathcal{C}_r$\;
    \ForEach{$c \in \mathcal{C}_r$ in parallel}{
        $v_c \leftarrow G_c(\mathcal{D}_c, w_{r-1})$\; 
    }
    $\mathcal{L}_r \leftarrow \texttt{SelectFunction}(\mathcal{C}_r, \lfloor f_1 C \rfloor, \{v_c\}_{c\in\mathcal{C}_r}, u)$\;
    \ForEach{$c \in \mathcal{L}_r$ in parallel}{
        $w_r^c \leftarrow \texttt{ClientUpdate}(w_{r-1}, \eta, B, E, \beta_1, \beta_2, \epsilon, \lambda)$\;
    }
    $w_r \leftarrow \sum_{c \in \mathcal{L}_r} \frac{K_c}{\tilde{K}_r}\, w_r^c$, where $\tilde{K}_r = \sum_{c \in \mathcal{L}_r} K_c$\;
}
\Return $w_R$\;
\end{algorithm}

\paragraph{NLL as a selection metric}
We choose NLL as one selection metric for the following two reasons: 
Since NLL is incorporated as part of the loss function, a client has a higher NLL if the global model is not sufficiently trained with respect to its local data.
As the local data is non-iid covering different driving scenarios, by selecting clients with higher NLL, the global model in FL is encouraged to do more local training on clients with more difficult data. 

\paragraph{AU as a selection metric} 
This metric is inspired by \cite{xi2021robust}, where incremental active learning is adopted for human trajectory prediction to evaluate candidate data samples and then select more valuable samples. Specifically, both noisy and redundant trajectory candidate data samples are removed and the model trained on filtered data samples achieves better performance. In our FLTP, we exploit aleatoric uncertainty to measure the degree of data noise. High aleatoric uncertainty means data are very noisy, while low aleatoric uncertainty means data are easy and the model is certain about them. As a result, we prefer clients with median aleatoric uncertainty, as data on these clients are both representative and less noisy. 

\vskip 0.6\baselineskip
\noindent {\em Relaxing full client participation in updating $v$.} 
For ease of exposition, in Algorithm \ref{alg:alfltp}, 
we let every client participate in updating $v_c$. In practice, it suffices to have the clients in $\mathcal{C}_r$ do the value updates only. Since the value update does not rely on any previous value of $v_c$,
the updated values are only used in the sorting at the PS, and the PS knows the $\mathcal{C}_r$, the PS can treat $v_c=0$ for all $c\notin \mathcal{C}_r$.

\subsection{Uncertainty-Awareness and Heterogeneous-Complexity Client Selector}
\label{subsec: complexity selector}
As we have introduced previously, FLTP-NLL or FLTP-AU solely uses NLL or AU as a standalone client selection criterion. Although these single-factor criteria significantly improve performance by leveraging uncertainty measures, real-world traffic scenes often also exhibit a wide range of heterogeneity in scene complexity. 
Thus, to further enhance model robustness and generalization, we introduce {\our}, an advanced adaptive client selection mechanism integrating multiple factors:  complexity ranking and training loss.

\subsubsection{ Complexity Quantification}
\label{sec: data partition}

We filter the dataset by city (e.g., Pittsburgh or Miami) to preserve geographical consistency within each client’s dataset. Each 
driving scene, denoted as \( S_i \), is then processed to extract a set of interpretable features that reflect the complexity of its driving scenario. These features include the number of left turns, the average vehicle speed, the average distance to nearby intersections, and the density of surrounding agents, inferred from the number of unique \texttt{TRACK\_ID} objects in the scene.

The number of left turns in a trajectory is associated with increased decision-making complexity. Average speed reflects the dynamics of the driving scene—trajectories with abrupt or high-speed segments often indicate more challenging prediction tasks. The intersection distance measures how close a trajectory is to traffic junctions, where smaller distances correspond to more interaction-heavy or congested regions. Population density provides a proxy for environmental complexity, capturing how many other agents are present near the ego vehicle.

For each \(S_i\), we extract four interpretable features that correlate with scene complexity:
\noindent
$t_i$: \# left turns;\;
$v_i$: avg.\ speed;\;
$d_i$: mean dist.\ to nearest intersection;\;
$\lambda_i$: agent density (unique \texttt{TRACK\_ID}s).
Let \(s\in\mathbb{N}\) be a normalization scale (typically \(s{=}9\)). For any feature family \(f_k\in\{t,v,d,\lambda\}\), define its min–max normalization across the corpus
\[
\tilde f_k(S_i)\;=\;s\,
\frac{f_k(S_i)-\min_j f_k(S_j)}{\max_j f_k(S_j)-\min_j f_k(S_j)}.
\]
We then form a \emph{complexity index} as a continuous score
\begin{equation}
\label{eq:kappa}
\kappa(S_i)\;=\;\tfrac{1}{4}\Big(\tilde f_{t}(S_i)+\tilde f_{v}(S_i)+\big(s-\tilde f_{d}(S_i)\big)+\tilde f_{\lambda}(S_i)\Big),
\end{equation}
where the intersection term is inverted so that smaller intersection distances (larger interaction likelihood) raise complexity.%

To obtain discrete levels, fix the number of bins \(\ell\in\mathbb{N}\) (e.g., \(\ell=s{+}1\)) and compute empirical percentiles \(\{q_0,\dots,q_\ell\}\) over \(\{\kappa(S_i)\}\) with \(q_0=\min\kappa\), \(q_\ell=\max\kappa\). 
The \emph{rank}
\begin{equation}
\label{eq:rho}
\begin{aligned}
\rho(S_i)\;&=\;\min\{j\in\{1,\dots,\ell\}:\kappa(S_i)\le q_j\}-1 \\
& \in\{0,\dots,\ell-1\}    
\end{aligned} 
\end{equation}
coarsely stratifies trajectories from easiest (\(0\)) to most complex (\(\ell{-}1\)). Combined with Dirichlet partitioning (concentration \(\alpha\)), this yields realistic non-IID client datasets with interpretable per-sample labels and supports our later use of complexity-
aware client selection in federated learning.

\subsubsection{Selection Criteria}

Formally, each client $c \in \mathcal{C}$ computes two metrics. First, the trajectory  Complexity rank is calculated by averaging the  complexity scores of local trajectories:
\begin{equation}
\bar{\kappa}_c\;\triangleq\;\frac{1}{K_c}\sum_{k=1}^{K_c}\kappa\!\big(S_c^k\big),
\end{equation} 
where $K_c$ represents the number of trajectories at client $c$, and $\kappa(S^k)$ denotes the complexity metric derived from either handcrafted trajectory features. 

Second, the training loss for each client is computed as:
\begin{equation}
L_c(w)\;\triangleq\;\frac{1}{K_c}\sum_{k=1}^{K_c} L\!\big(S_c^k,\,w\big),
\end{equation}
recalling that $L(S^k, w)$ is the loss evaluated on the scene $S^k$ using the current global model parameters $w$.

\noindent\textbf{Deterministic hybrid.}
Let $|\mathcal{L}_r|$ denote the total number of clients selected in round $r$, and let $\nu \in [0,\,|\mathcal{L}_r|]$ be a configurable parameter specifying how many clients to select based on scene complexity ranking. Specifically, {\our} selects the top $\nu$ clients with the most complex local datasets (as quantified by trajectory rank scores) and an additional $(|\mathcal{L}_r|-\nu)$ clients exhibiting the highest training loss. When $\nu=0$, the selection prioritizes model underperformance exclusively (pure loss-driven); when $\nu=|\mathcal{L}_r|$, it selects clients solely by data complexity. Intermediate values of $\nu$ balance these two priorities.
In round $r$, FLTP\text{-}DR selects clients in two stages to balance challenging data and underperforming regions:
\[
\mathcal{L}_r^{(\kappa)}=\text{Top-}\nu\ \text{by }\bar{\kappa}_c,\qquad
\mathcal{L}_r^{(\text{loss})}=\text{Top-}\big(|\mathcal{L}_r|-\nu\big)\ ,
\]
\[
\mathcal{L}_r=\mathcal{L}_r^{(\kappa)}\cup\mathcal{L}_r^{(\text{loss})}.
\]

All selections are without replacement within round $r$.

\noindent\textbf{Probabilistic variant.}
To promote broader participation and reduce selection bias, we replace hard top-$k$ with sampling over the candidate pool $\mathcal{C}_r = \mathcal{C}$.
Z\mbox{-}score each statistic within $\mathcal{C}$ to remove scale effects,
\[
\begin{aligned}
\tilde{\kappa}_c = \frac{\bar{\kappa}_c-\mu_\kappa}{\sigma_\kappa}, ~~~ \text{and} ~~~ 
\tilde{L}_c      = \frac{L_c(w_{r-1})-\mu_L}{\sigma_L}, 
\end{aligned}
\] 
where $(\mu_\kappa,\sigma_\kappa)$ and $(\mu_L,\sigma_L)$ are the mean and standard deviation over $c\in\mathcal{C}$, $\bar{\kappa}_c$ is the client’s average complexity, and $L_c(w_{r-1})$ is the client’s average loss.

Then we map scores to nonnegative weights and normalize to probabilities:
\[
\begin{aligned}
p_c^{(\kappa)} &= \frac{\max\{\tilde{\kappa}_c,0\}}
{\sum_{c'\in\mathcal{C}_r}\max\{\tilde{\kappa}_{c'},0\}},\\
p_c^{(\text{loss})} &= \frac{\max\{\tilde{L}_c,0\}}
{\sum_{c'\in\mathcal{C}_r}\max\{\tilde{L}_{c'},0\}}.
\end{aligned}
\]
Then we sample \emph{without replacement} $\nu$ clients from $\mathcal{C}$ according to $p^{(\kappa)}$ and the remaining $(|\mathcal{L}_r|]-\nu)$ according to $p^{(\text{loss})}$, removing any client drawn in the first stage before the second and renormalizing. 
This introduces diversity while preserving emphasis on high\mbox{-}complexity and high\mbox{-}loss clients; if a denominator is zero, fall back to uniform sampling over the remaining pool.

Each selected client receives the current global model weights $w_{r-1}$ and performs local updates using the AdamW optimizer on its dataset for $E$ epochs. After local training, updated client models are transmitted back to the server for aggregation. The server then aggregates these client updates in a weighted manner based on client dataset sizes 

\begin{equation}
    w_{r} = \sum_{c \in \mathcal{L}_r}
    \frac{K_c}{\tilde{K}_r} w_r^{c},
\end{equation}

The iterative training process continues for a predefined number of global rounds or until convergence is reached. By explicitly incorporating both scene complexity and training loss into client selection and aggregation strategies, {\our} effectively addresses the challenges posed by non-IID federated data, enhancing robustness, privacy, and predictive accuracy.
The complete workflow of {\our} is summarized in Algorithms~\ref{alg:fl_training}, detailing the complexity-aware scoring mechanism and federated training steps, respectively.

\begin{algorithm}[ht]
\caption{\small FL with complexity–Ranked Adaptive Selection ({\our})}
\label{alg:fl_training}
\small 
\KwIn{Initial model $w_0$, Dataset $\mathcal{D}$, number of clients $C$,  stepsize $\eta$, batch size $B$, number of epochs $E$, weight decay $\lambda$, momentum parameters $\beta_1, \beta_2 \in (0,1)$, smoothing parameter $\epsilon$,  target size $|\mathcal{L}_r|$, selection balance parameter $\nu$}
\KwOut{Final model $w_R$}

Initialize global model: $w \leftarrow w_0$\;

\For{each client $c\in \mathcal{C}$}
{Computes and reports $\bar{\kappa}_c$ -- the complexity score of its local dataset. }

\For{$r=1$ \KwTo $R$}{
  Send $w_{r-1}$ to all clients $\mathcal{C}$ \;

  \ForEach{$c\in\mathcal{C}$ in parallel}{
    $\displaystyle L_c(w_{r-1})\leftarrow \frac{1}{K_c}\sum_{S_c^k\in\mathcal{D}_c} L(S_c^k,\,w_{r-1})$\;
  }

  $\mathcal{L}_r^{(\kappa)} \leftarrow \texttt{SelectFunction}(\mathcal{C}, \{\bar{\kappa}_c\}, \nu, \textsf{DR})$\;

  $\mathcal{L}_r^{(\text{loss})} \leftarrow \texttt{SelectFunction}(\mathcal{C}\setminus\mathcal{L}_r^{(\kappa)}, \{ L_c(w_{r-1})\}_{c\in\mathcal{C}\setminus \mathcal{L}_r^{(\kappa)}}, |\mathcal{L}_r|, \textsf{NLL})$\;
  
  $\mathcal{L}_r \leftarrow \mathcal{L}_r^{(\kappa)} \cup \mathcal{L}_r^{(\text{loss})}$\;

  \ForEach{$c\in\mathcal{L}_r$ in parallel}{
    $w_r^c \leftarrow \texttt{ClientUpdate}(w_{r-1}, \eta, B, E, \beta_1, \beta_2, \epsilon, \lambda)$\;
  }

  $w_r \leftarrow \sum_{c \in \mathcal{L}_r} \frac{K_c}{\tilde{K}_r}\, w_r^c$, where $\tilde{K}_r = \sum_{c \in \mathcal{L}_r} K_c$\;
}
\Return $w_R$\;

\end{algorithm}

\subsection{Theoretical Analysis}
\label{subsec: theorem}
We provide theoretical support for our client selection from two
perspectives. First, we give a formal convergence statement under a
decaying gradient-bias assumption. Second, we use an importance-sampling
interpretation to explain why non-uniform DR-based sampling can lead to superior performance.  

\subsubsection{Convergence Analysis}
Our analysis follows the standard stochastic first-order optimization
framework for nonconvex objectives, where convergence is characterized
through the gradient norm 
$\mathbb{E}\|\nabla F(w_r)\|^2$. 
Let
\begin{equation}
F(w)=\frac{1}{C}\sum_{c=1}^C F_c(w)
\end{equation}
be the global objective, where $F_c$ is the local trajectory prediction loss
of client $c$. 
At communication round $r$, the server update is
\begin{equation}
w_{r+1}=w_r-\eta g_r,
\end{equation}
where $g_r$ is the aggregated update determined by client selection.

\begin{assumption}[Smoothness]
\label{ass:smooth}

The global objective $F:\mathbb{R}^d\rightarrow\mathbb{R}$
is $L$-smooth, i.e., for any $u,v\in\mathbb{R}^d$,
\begin{equation}
F(v)
\leq
F(u)
+
\langle \nabla F(u),v-u\rangle
+
\frac{L}{2}\|v-u\|^2 .
\end{equation}

\end{assumption}

\begin{assumption}[Gradient bias and noise]
\label{ass:bias_noise}
For each $r$, define
\begin{equation}
\delta_r
:=
\left\|
\mathbb{E}[g_r\mid w_r]-\nabla F(w_r)
\right\|.
\end{equation}
We assume the DR-induced bias decays as
\begin{equation}
\delta_r \leq D/r^\epsilon, 
\qquad D>0,\quad \epsilon>0.
\end{equation}
We also assume the stochastic error satisfies
\begin{equation}
\mathbb{E}
\left[
\|g_r-\nabla F(w_r)\|^2
\mid w_r
\right]
\leq
\sigma^2+\delta_r^2 .
\end{equation}
\end{assumption}

Figure~\ref{fig:gradient_error} empirically illustrates the decreasing trend of 
$\|e_r\|=\|g_r-\nabla F(w_r)\|$, which is consistent with the decaying-bias 
condition used in Assumption~\ref{ass:bias_noise}.

\begin{figure}[t]
    \centering
    \includegraphics[width=0.8\columnwidth]{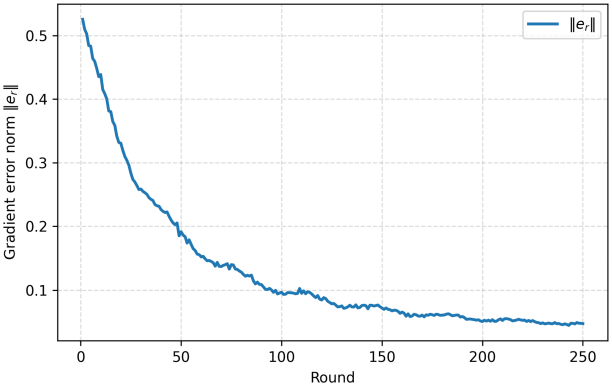}
    \caption{Evolution of the gradient error norm 
    $\|e_r\|=\|g_r-\nabla F(w_r)\|$ over communication rounds.The decreasing trend indicates that the discrepancy between the 
    selected-client update and the full gradient gradually diminishes 
    during training.}
    \label{fig:gradient_error}
\end{figure}

\begin{lemma}[One-step descent under biased updates]
\label{lem:biased_descent}
Suppose Assumptions~\ref{ass:smooth}--\ref{ass:bias_noise} hold.
If $\eta\leq 1/(4L)$, then
\begin{align*}
\mathbb{E}[F(w_{r+1})\mid w_r]
&\leq
F(w_r)
-
\frac{\eta}{4}
\|\nabla F(w_r)\|^2 \\
&\quad+
\left(
\frac{\eta}{2}+L\eta^2
\right)\delta_r^2
+
L\eta^2\sigma^2 .
\end{align*} 
\end{lemma}

\begin{proof}
By $L$-smoothness in Assumption \ref{ass:smooth},
\begin{align*}
F(w_{r+1})
&\leq
F(w_r)
-\eta\langle \nabla F(w_r),g_r\rangle
+\frac{L\eta^2}{2}\|g_r\|^2 .
\end{align*}
Let $e_r=g_r-\nabla F(w_r).$
Then
\begin{align*}
\langle \nabla F(w_r),g_r\rangle
&=
\|\nabla F(w_r)\|^2
+
\langle \nabla F(w_r),e_r\rangle .
\end{align*}
By the Cauchy--Schwarz inequality and Assumption~\ref{ass:bias_noise}, we obtain  
\begin{align*}
&|\langle \nabla F(w_r),\mathbb{E}
[e_r \mid w_r] \rangle | 
\le \left|
\langle \nabla F(w_r), \mathbb{E}
[e_r 
\mid w_r
]
\rangle 
\right| \\
&\qquad \qquad \qquad \leq
\|\nabla F(w_r)\|\delta_r 
\leq
\frac{1}{2}\|\nabla F(w_r)\|^2
+
\frac{1}{2}\delta_r^2,
\end{align*}
where the last step follows from Young's inequality. 
Thus, 
\begin{align*}
\langle \nabla F(w_r),g_r\rangle 
&\ge \|\nabla F(w_r)\|^2 - \frac{1}{2}\|\nabla F(w_r)\|^2 - 
\frac{1}{2}\delta_r^2 \\
&= \frac{1}{2}\|\nabla F(w_r)\|^2 - 
\frac{1}{2}\delta_r^2. 
\end{align*}
Therefore,
\begin{align*}
\mathbb{E}[F(w_{r+1})\mid w_r]
&\leq
F(w_r)
-
\frac{\eta}{2}
\|\nabla F(w_r)\|^2 \\
&\quad+
\frac{\eta}{2}\delta_r^2
+
\frac{L\eta^2}{2}
\mathbb{E}[\|g_r\|^2\mid w_r].
\end{align*}
In addition, $\|g_r\|^2 \le 2\|\nabla F(w_r)\|^2
+
2\|g_r-\nabla F(w_r)\|^2$.  
Taking conditional expectation and applying Assumption~\ref{ass:bias_noise},
\begin{align*}
\mathbb{E}[\|g_r\|^2\mid w_r]
&\leq
2\|\nabla F(w_r)\|^2
+
2(\sigma^2+\delta_r^2).
\end{align*}
Substituting this bound gives
\begin{align*}
\mathbb{E}[F(w_{r+1})\mid w_r]
&\leq
F(w_r)
-
(
\frac{\eta}{2}-L\eta^2
)
\|\nabla F(w_r)\|^2 \\
&\quad+
(
\frac{\eta}{2}+L\eta^2
)\delta_r^2
+
L\eta^2\sigma^2 .
\end{align*}
Since $\eta\leq 1/(4L)$, we have $\frac{\eta}{2}-L\eta^2\geq \frac{\eta}{4}.$
\end{proof}

\begin{theorem}[Convergence under decaying bias]
\label{thm:dr_convergence}
Suppose Assumptions~\ref{ass:smooth}--\ref{ass:bias_noise} hold and
$F(w)\geq F^\star$. If $\eta\leq 1/(4L)$, then after $R$ communication rounds,
\begin{align*}
\frac{1}{R}
\sum_{r=1}^R
\mathbb{E}\|\nabla F(w_r)\|^2
&\leq
\frac{4(F(w_1)-F^\star)}{\eta R} \\
&\quad+
(2+4L\eta)
\frac{1}{R}
\sum_{r=1}^R
\delta_r^2
+
4L\eta\sigma^2 .
\end{align*}
\end{theorem}
%

\begin{proof}
Taking total expectation in Lemma~\ref{lem:biased_descent} and summing
from $r=1$ to $R$ yields
\begin{align*}
\frac{\eta}{4}
\sum_{r=1}^R
\mathbb{E}\|\nabla F(w_r)\|^2
&\leq
F(w_1)-\mathbb{E}[F(w_{R+1})] \\
&\quad+
\left(
\frac{\eta}{2}+L\eta^2
\right)
\sum_{r=1}^R\delta_r^2
+
L\eta^2R\sigma^2 .
\end{align*}
Using $F(w_{R+1})\geq F^\star$ and dividing by $\eta R/4$ gives
\begin{align*}
\frac{1}{R}
\sum_{r=1}^R
\mathbb{E}\|\nabla F(w_r)\|^2
&\leq
\frac{4(F(w_1)-F^\star)}{\eta R} +
(2+4L\eta)
\frac{1}{R}
\sum_{r=1}^R\delta_r^2 \\
&\quad
+4L\eta\sigma^2 .
\end{align*}
Finally, using $\delta_r\leq D/r^\epsilon$,
$\frac{1}{R}
\sum_{r=1}^R \delta_r^2
\leq
\frac{D^2}{R}
\sum_{r=1}^R r^{-2\epsilon}.$  
\end{proof}

Moreover, if $\delta_r\leq D/r^\epsilon$, then
\begin{equation*}
\begin{aligned}
\frac{1}{R}
\sum_{r=1}^R
\mathbb{E}\|\nabla F(w_r)\|^2
\leq\;
&\mathcal{O}\!\left(\frac{1}{\eta R}\right)
+
\mathcal{O}(\eta\sigma^2) \\
&+
\mathcal{O}\!\left(
\frac{1}{R}
\sum_{r=1}^R r^{-2\epsilon}
\right), 
\end{aligned}
\end{equation*}
where $\mathcal{O}$ only hides absolute constants.

The bias accumulation term in Theorem \ref{thm:dr_convergence} satisfies
\begin{equation}
\frac{1}{R}
\sum_{r=1}^R r^{-2\epsilon}
=
\begin{cases}
\mathcal{O}(R^{-2\epsilon}), & 0<2\epsilon<1,\\
\mathcal{O}(\log R/R), & 2\epsilon=1,\\
\mathcal{O}(1/R), & 2\epsilon>1.
\end{cases}
\end{equation}
Thus, if the active client selection-induced gradient bias decays over training, the average
gradient norm remains controlled. In the standard nonconvex stochastic
optimization sense, the iterates approach a stationary point up to stochastic
noise and the residual bias-decay term.
In particular, when $\epsilon = 1/2$ and
$\eta = \min{1/(4L),1/\sqrt{R}}$,
the bound consists of three components:
(i) the optimization term $\mathcal{O}(1/\sqrt{R})$,
(ii) the stochastic-noise term $\mathcal{O}(\sigma^2/\sqrt{R})$,
and (iii) the bias-decay contribution
$\mathcal{O}(\log R / R)$.
As $R$ increases, both the optimization and bias terms vanish,
leaving only the standard stochastic-noise floor.
Therefore, the proposed client-selection strategy preserves the
convergence behavior of stochastic first-order optimization while
allowing a controlled, decaying selection bias~\cite{10.1007/s10107-022-01822-7}.

\subsubsection{Importance-Sampling Intepretation}
We now provide an importance-sampling interpretation explaining why
non-uniform (particularly DR-based non-uniform) client selection may reduce update variance and
support stable learning.

At round $r$, let $s_{c,r}\geq 0$ denote the DR score of client $c$.
Define the proposal distribution
\begin{equation}
q_{c,r}
=
\frac{s_{c,r}}
{\sum_{c=1}^C s_{j,r}}.
\end{equation}
This proposal assigns a larger sampling probability to clients with
larger DR scores, corresponding to more difficult or informative
trajectory scenarios.


Define the reweighted estimator
\begin{equation}
g_r^{\mathrm{IS}}
=
\frac{1}{C}
\sum_{c\in \mathcal{C}_r}
\frac{p_k}{q_{c,r}}g_{c,r},
\end{equation}
where $S_r$ is sampled according to $q_{c,r}$.
Then
\begin{equation}
\mathbb{E}_{q_r}[g_r^{\mathrm{IS}}]
=\frac{1}{C}
\sum_{c=1}^C  g_{c,r}.
\end{equation}
That is, our method can be viewed as a valid importance-sampling scheme, which is shown to significantly reduce variance \cite{alain2015variance}, thereby eventually leading to superior performance. 



\section{Evaluation Results}

\subsection{Experimental Setup}
\noindent \textbf{Dataset:} We use the Argoverse Motion Forecasting v1.1 dataset for training and evaluation, 
with 95,521 samples are from Pittsburgh, and 110,421 samples from Miami. 
The validation set contains 39,472 samples. 
All training and validation scenarios consist of 5-second trajectory segments, sampled at 10 Hz, along with map information. The Argoverse Motion Forecasting challenge is to predict future trajectories of 3 seconds of focal agents, using the past 2-second trajectories 
as inputs.


\noindent \textbf{Models and Training Parameters:} 
We use HiVT-64 \cite{zhou2022hivt} as our trajectory prediction model to evaluate the performance of our uncertainty-aware client selectors. We chose HiVT because, at the time of the preliminary version of our method, it represented the state of the art and was the first to capture global agent–agent interactions in traffic scenes.    
The local HiVT models are trained with the same parameters as the centralized HiVT. Specifically, for each local model in FLTP, we set the learning rate ($\eta$) to $5 \times 10^{-4}$, weight decay to $1 \times 10^{-4}$, dropout rate to 0.1, local batch size ($B$) to 32, local epochs ($E$) to 4, and the optimizer to AdamW. 
We train the model for 250 rounds. The fraction of clients selected for communication in each round ($f_1$) is set to be 0.1. 
We conduct experiments with two different random seeds and show the averaged results.

In our {\our} framework, we employ the \textbf{HPNet} model for trajectory prediction. HPNet is a hierarchical prediction network designed to capture both global and local dependencies in observed trajectories. The model takes as input a sequence of observed positions and outputs predicted trajectories for future timesteps. Key aspects of HPNet include hierarchical feature encoding, multi-modal prediction, and robust optimization techniques. HPNet is optimized using the Adam optimizer with a learning rate of $3 \times 10^{-4}$ and weight decay of $1 \times 10^{-4}$. The training is performed with a batch size of 32 over 250 global rounds. 
%
We compare {\our} against several baseline methods using a fixed client selection ratio of $f_1 = 30\%$. The notation and description for each baseline are summarized clearly in Table~\ref{tab:baselines}.

\begin{table}[!t]
\footnotesize
\renewcommand{\arraystretch}{1.15}
\setlength{\tabcolsep}{3pt}
\caption{Summary of baseline methods used in evaluation.}
\label{tab:baselines}
\centering
\begin{tabular*}{\columnwidth}{@{\extracolsep{\fill}} l p{0.7\columnwidth}}
\toprule
\textbf{Method} & \textbf{Description} \\
\midrule
Random & Clients are selected randomly at each training round. \\
High Loss & The top $f_1$ fraction (30\%) of clients with the highest training cross-entropy loss and Negative Log-Likelihood (NLL) are selected. \\
FLTP-NLL & Clients are selected based solely on Negative Log-Likelihood (NLL), without explicitly considering rank or loss separately. \\
High Rank + Random & The top-ranked $\nu$ clients are selected based on  complexity scores, and the remaining $(f_1C - \nu)$ clients are selected randomly. \\
{\our} & Clients are selected using a configurable parameter $\nu$ balancing between high-rank (trajectory  complexity) and high-loss criteria, totaling $f_1 = 30\%$ of all clients. \\
\bottomrule
\end{tabular*}
\end{table}

\noindent \textbf{Evaluation Metrics:} We use NLL, Minimum Average Displacement Error (minADE), Minimum Final Displacement Error (minFDE) and Miss Rate (MR) to evaluate model performance quantitatively. minADE measures the average L2 distance between the best predicted trajectory (the trajectory with the minimum error at the endpoint) and the ground truth. minFDE measures the L2 distance at the endpoint between the best predicted trajectory and the ground truth. MR quantifies the fraction of scenarios where the endpoint errors of all predicted trajectories exceed 2 meters.

\subsection{FLTP v.s. Training on Local Data}
We quantitatively compare the global model of FLTP with the local model of an arbitrarily chosen client when it does not participate in communication and only updates using its local data. In this case, we have chosen client 0; selecting any other client would yield similar results. As shown in Fig. \ref{fig:wofl} and Table \ref{i^*b:roundshotdata}, FLTP significantly outperforms the client without FL, demonstrating the effectiveness of FLTP in leveraging multi-source traffic data, even though it does not explicitly access raw local data. Specifically, after about 50 rounds, the local model of client 0 begins to exhibit worse performance as the number of training rounds increases, indicating that the local model without FL has poor generalization.

\begin{figure*}[ht]
\centering
\begin{subfigure}{0.27\linewidth}
\begin{tikzpicture}
\begin{axis}[
    xlabel={\scriptsize{Round}},
    ylabel={\scriptsize{$L_{reg}$}},
    every axis x label/.style={at={(current axis.south)},below=8pt},
    every axis y label/.style={
            at={(ticklabel* cs:0.5)},rotate=90,anchor=center,align=center, above=12.5pt},
    xmin=0, xmax=250,
    ymin=-0.5, ymax=0,
    xtick={0,50,100,150,200,250},
    xticklabels={\scriptsize{0},\scriptsize{50},\scriptsize{100},\scriptsize{150},\scriptsize{200},\scriptsize{250}},
    ytick={-0.5,-0.4,-0.3,-0.2,-0.1,0},
    yticklabels={\scriptsize{-0.5},\scriptsize{-0.4},\scriptsize{-0.3},\scriptsize{-0.2},\scriptsize{-0.1},\scriptsize{0}},
    height = 1\linewidth,
    width = 1\linewidth,
    xmajorgrids=true,
    ymajorgrids=true,
    grid style=dashed,
    legend cell align={left},
    legend style={inner sep=0pt,row sep=-3pt},
    legend style={font=\tiny},
    every axis plot/.append style={ultra thick}
]
    
    \addplot [
    color=Gray,
    line width=1pt
    ]
    table[x=roundshow,y=flloss,col sep = comma] {iros_data_noniid.csv};

    \addplot [
    color=Salmon,
    line width=1pt,
    ]
    table[x=roundshow,y=localloss,col sep = comma] {iros_data_noniid.csv};

\legend{FLTP, Client 0 w/o FL}

\end{axis}
\end{tikzpicture}
\vspace{-5pt}
\end{subfigure}
\hspace{-20pt}
\begin{subfigure}{0.27\linewidth}
\begin{tikzpicture}
\begin{axis}[
    xlabel={\scriptsize{Round}},
    ylabel={\scriptsize{minADE}},
    every axis x label/.style={at={(current axis.south)},below=8pt},
    every axis y label/.style={
            at={(ticklabel* cs:0.5)},rotate=90,anchor=center,align=center, above=13pt},
    xmin=0, xmax=250,
    ymin=0.7, ymax=1.5,
    xtick={0,50,100,150,200,250},
    xticklabels={\scriptsize{0},\scriptsize{50},\scriptsize{100},\scriptsize{150},\scriptsize{200},\scriptsize{250}},
    ytick={0.7,0.9,1.1,1.3,1.5},
    yticklabels={\scriptsize{0.7},\scriptsize{0.9},\scriptsize{1.1},\scriptsize{1.3},\scriptsize{1.5}},
    height = 1\linewidth,
    xmajorgrids=true,
    ymajorgrids=true,
    grid style=dashed,
    legend cell align={left},
    legend style={inner sep=0pt,row sep=-3pt},
    legend style={font=\tiny},
    every axis plot/.append style={ultra thick}
]
    
    \addplot [
    color=Gray,
    line width=1pt,
    ]
    table[x=roundshow,y=flade,col sep = comma] {iros_data_noniid.csv};

    \addplot [
    color=Salmon,
    line width=1pt,
    ]
    table[x=roundshow,y=localade,col sep = comma] {iros_data_noniid.csv};

\legend{FLTP, Client 0 w/o FL}
    
\end{axis}
\end{tikzpicture}
\vspace{-5pt}
\end{subfigure}
\hspace{-20pt}
\begin{subfigure}{0.27\linewidth}
\begin{tikzpicture}
\begin{axis}[
    xlabel={\scriptsize{Round}},
    ylabel={\scriptsize{minFDE}},
    every axis x label/.style={at={(current axis.south)},below=8pt},
    every axis y label/.style={
            at={(ticklabel* cs:0.5)},rotate=90,anchor=center,align=center, above=10pt},
    xmin=0, xmax=250,
    ymin=1.1, ymax=2.3,
    xtick={0,50,100,150,200,250},
    xticklabels={\scriptsize{0},\scriptsize{50},\scriptsize{100},\scriptsize{150},\scriptsize{200},\scriptsize{250}},
    ytick={1.1,1.3,1.5,1.7,1.9,2.1,2.3},
    yticklabels={\scriptsize{1.1},\scriptsize{1.3},\scriptsize{1.5},\scriptsize{1.7},\scriptsize{1.9},\scriptsize{2.1},\scriptsize{2.3}},
    height = 1\linewidth,
    xmajorgrids=true,
    ymajorgrids=true,
    grid style=dashed,
    legend cell align={left},
    legend style={inner sep=0pt,row sep=-3pt},
    legend style={font=\tiny},
    every axis plot/.append style={ultra thick}
]
    
    \addplot [
    color=Gray,
    line width=1pt,
    ]
    table[x=roundshow,y=flfde,col sep = comma] {iros_data_noniid.csv};

    \addplot [
    color=Salmon,
    line width=1pt,
    ]
    table[x=roundshow,y=localfde,col sep = comma] {iros_data_noniid.csv};

\legend{FLTP, Client 0 w/o FL}
    
\end{axis}
\end{tikzpicture}
\vspace{-5pt}
\end{subfigure}
\hspace{-20pt}
\begin{subfigure}{0.27\linewidth}
\begin{tikzpicture}
\begin{axis}[
    xlabel={\scriptsize{Round}},
    ylabel={\scriptsize{MR}},
    every axis x label/.style={at={(current axis.south)},below=8pt},
    every axis y label/.style={
            at={(ticklabel* cs:0.5)},rotate=90,anchor=center,align=center, above=13pt},
    xmin=0, xmax=250,
    ymin=0.1, ymax=0.5,
    xtick={0,50,100,150,200,250},
    xticklabels={\scriptsize{0},\scriptsize{50},\scriptsize{100},\scriptsize{150},\scriptsize{200},\scriptsize{250}},
    ytick={0.1,0.2,0.3,0.4,0.5},
    yticklabels={\scriptsize{0.1},\scriptsize{0.2},\scriptsize{0.3},\scriptsize{0.4},\scriptsize{0.5}},
    height = 1\linewidth,
    xmajorgrids=true,
    ymajorgrids=true,
    grid style=dashed,
    legend cell align={left},
    legend style={inner sep=0pt,row sep=-3pt},
    legend style={font=\tiny},
    every axis plot/.append style={ultra thick}
]
    
    \addplot [
    color=Gray,
    line width=1pt,
    ]
    table[x=roundshow,y=flmr,col sep = comma] {iros_data_noniid.csv};

    \addplot [
    color=Salmon,
    line width=1pt,
    ]
    table[x=roundshow,y=localmr,col sep = comma] {iros_data_noniid.csv};

\legend{FLTP, Client 0 w/o FL}
    
\end{axis}
\end{tikzpicture}
\vspace{-5pt}
\end{subfigure}
\caption{Round-wise comparison between FLTP and the local model of client 0 without FL. Fraction of clients selected for communication in each round $f_1$ is set to be 0.1.}
\label{fig:wofl}
\vspace{-5pt}
\end{figure*}
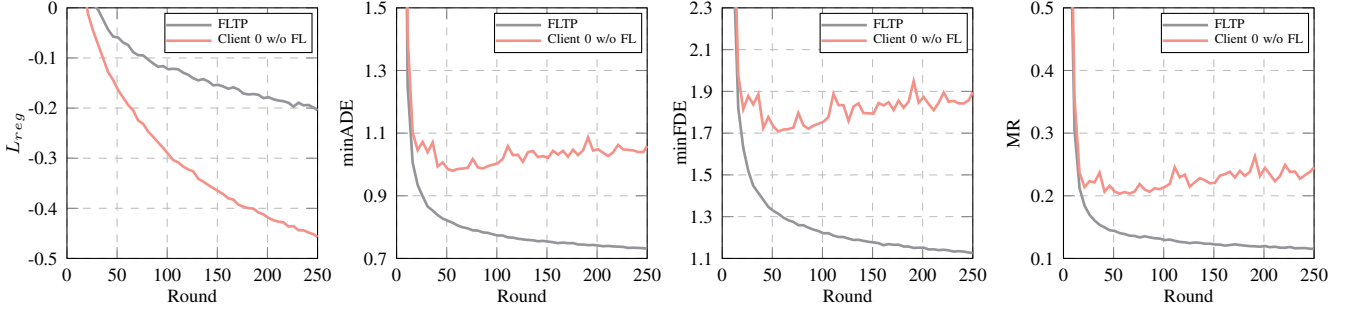

\begin{figure*}[ht]
\centering
\begin{subfigure}{0.27\linewidth}
\begin{tikzpicture}
\begin{axis}[
    xlabel={\scriptsize{Round}},
    ylabel={\scriptsize{$L_{reg}$}},
    every axis x label/.style={at={(current axis.south)},below=8pt},
    every axis y label/.style={
            at={(ticklabel* cs:0.5)},rotate=90,anchor=center,align=center, above=12.5pt},
    xmin=0, xmax=250,
    ymin=-0.22, ymax=0,
    xtick={0,50,100,150,200,250},
    xticklabels={\scriptsize{0},\scriptsize{50},\scriptsize{100},\scriptsize{150},\scriptsize{200},\scriptsize{250}},
    ytick={-0.2,-0.1,0,0.1,0.2,0.3,0.4,0.5,0.6},
    yticklabels={\scriptsize{-0.2},\scriptsize{-0.1},\scriptsize{0},\scriptsize{0.1},\scriptsize{0.2},\scriptsize{0.3},\scriptsize{0.3},\scriptsize{0.4},\scriptsize{0.5},\scriptsize{0.6}},
    height = 1\linewidth,
    width = 1\linewidth,
    xmajorgrids=true,
    ymajorgrids=true,
    grid style=dashed,
    legend cell align={left},
    legend style={inner sep=0pt,row sep=-3pt},
    legend style={font=\tiny},
    every axis plot/.append style={ultra thick}
]
    
    \addplot [
    color=Gray,
    line width=1pt
    ]
    table[x=roundshow,y=flloss,col sep = comma] {iros_data_noniid.csv};

    \addplot [
    color=BurntOrange,
    line width=1pt,
    ]
    table[x=roundshow,y=flalloss0.15,col sep = comma] {iros_data_noniid.csv};

    \addplot [
    color=Goldenrod,
    line width=1pt,
    ]
    table[x=roundshow,y=flalloss0.3,col sep = comma] {iros_data_noniid.csv};

\legend{FLTP,FLTP-NLL($f_2$=0.15),FLTP-NLL($f_2$=0.30)}

\end{axis}
\end{tikzpicture}
\vspace{-5pt}
\end{subfigure}
\hspace{-20pt}
\begin{subfigure}{0.27\linewidth}
\begin{tikzpicture}
\begin{axis}[
    xlabel={\scriptsize{Round}},
    ylabel={\scriptsize{minADE}},
    every axis x label/.style={at={(current axis.south)},below=8pt},
    every axis y label/.style={
            at={(ticklabel* cs:0.5)},rotate=90,anchor=center,align=center, above=13pt},
    xmin=0, xmax=250,
    ymin=0.72, ymax=0.85,
    xtick={0,50,100,150,200,250},
    xticklabels={\scriptsize{0},\scriptsize{50},\scriptsize{100},\scriptsize{150},\scriptsize{200},\scriptsize{250}},
    ytick={0.72,0.74,0.76,0.78,0.80,0.82,0.84},
    yticklabels={\scriptsize{0.72},\scriptsize{0.74},\scriptsize{0.76},\scriptsize{0.78},\scriptsize{0.80},\scriptsize{0.82},\scriptsize{0.84}},
    height = 1\linewidth,
    xmajorgrids=true,
    ymajorgrids=true,
    grid style=dashed,
    legend cell align={left},
    legend style={inner sep=0pt,row sep=-3pt},
    legend style={font=\tiny},
    every axis plot/.append style={ultra thick}
]
    
    \addplot [
    color=Gray,
    line width=1pt,
    ]
    table[x=roundshow,y=flade,col sep = comma] {iros_data_noniid.csv};

    \addplot [
    color=BurntOrange,
    line width=1pt,
    ]
    table[x=roundshow,y=flalade0.15,col sep = comma] {iros_data_noniid.csv};

    \addplot [
    color=Goldenrod,
    line width=1pt,
    ]
    table[x=roundshow,y=flalade0.3,col sep = comma] {iros_data_noniid.csv};



\legend{FLTP,FLTP-NLL($f_2$=0.15),FLTP-NLL($f_2$=0.30)}
    
\end{axis}
\end{tikzpicture}
\vspace{-5pt}
\end{subfigure}
\hspace{-20pt}
\begin{subfigure}{0.27\linewidth}
\begin{tikzpicture}
\begin{axis}[
    xlabel={\scriptsize{Round}},
    ylabel={\scriptsize{minFDE}},
    every axis x label/.style={at={(current axis.south)},below=8pt},
    every axis y label/.style={
            at={(ticklabel* cs:0.5)},rotate=90,anchor=center,align=center, above=10pt},
    xmin=0, xmax=250,
    ymin=1.1, ymax=1.4,
    xtick={0,50,100,150,200,250},
    xticklabels={\scriptsize{0},\scriptsize{50},\scriptsize{100},\scriptsize{150},\scriptsize{200},\scriptsize{250}},
    ytick={1.1,1.2,1.3,1.4},
    yticklabels={\scriptsize{1.1},\scriptsize{1.2},\scriptsize{1.3},\scriptsize{1.4}},
    height = 1\linewidth,
    xmajorgrids=true,
    ymajorgrids=true,
    grid style=dashed,
    legend cell align={left},
    legend style={inner sep=0pt,row sep=-3pt},
    legend style={font=\tiny},
    every axis plot/.append style={ultra thick}
]
    
    \addplot [
    color=Gray,
    line width=1pt,
    ]
    table[x=roundshow,y=flfde,col sep = comma] {iros_data_noniid.csv};

    \addplot [
    color=BurntOrange,
    line width=1pt,
    ]
    table[x=roundshow,y=flalfde0.15,col sep = comma] {iros_data_noniid.csv};

    \addplot [
    color=Goldenrod,
    line width=1pt,
    ]
    table[x=roundshow,y=flalfde0.3,col sep = comma] {iros_data_noniid.csv};

    

\legend{FLTP,FLTP-NLL($f_2$=0.15),FLTP-NLL($f_2$=0.30)}
    
\end{axis}
\end{tikzpicture}
\vspace{-5pt}
\end{subfigure}
\hspace{-20pt}
\begin{subfigure}{0.27\linewidth}
\begin{tikzpicture}
\begin{axis}[
    xlabel={\scriptsize{Round}},
    ylabel={\scriptsize{MR}},
    every axis x label/.style={at={(current axis.south)},below=8pt},
    every axis y label/.style={
            at={(ticklabel* cs:0.5)},rotate=90,anchor=center,align=center, above=13pt},
    xmin=0, xmax=250,
    ymin=0.11, ymax=0.15,
    xtick={0,50,100,150,200,250},
    xticklabels={\scriptsize{0},\scriptsize{50},\scriptsize{100},\scriptsize{150},\scriptsize{200},\scriptsize{250}},
    ytick={0.11,0.12,0.13,0.14,0.15},
    yticklabels={\scriptsize{0.11},\scriptsize{0.12},\scriptsize{0.13},\scriptsize{0.14},\scriptsize{0.15}},
    height = 1\linewidth,
    xmajorgrids=true,
    ymajorgrids=true,
    grid style=dashed,
    legend cell align={left},
    legend style={inner sep=0pt,row sep=-3pt},
    legend style={font=\tiny},
    every axis plot/.append style={ultra thick}
]
    
    \addplot [
    color=Gray,
    line width=1pt,
    ]
    table[x=roundshow,y=flmr,col sep = comma] {iros_data_noniid.csv};

    \addplot [
    color=BurntOrange,
    line width=1pt,
    ]
    table[x=roundshow,y=flalmr0.15,col sep = comma] {iros_data_noniid.csv};

    \addplot [
    color=Goldenrod,
    line width=1pt,
    ]
    table[x=roundshow,y=flalmr0.3,col sep = comma] {iros_data_noniid.csv};



\legend{FLTP,FLTP-NLL($f_2$=0.15),FLTP-NLL($f_2$=0.30)}
    
\end{axis}
\end{tikzpicture}
\vspace{-5pt}
\end{subfigure}
\caption{Round-wise comparison between FLTP and FLTP-NLL}
\label{fig:alfltpnll}
\vspace{-5pt}
\end{figure*}
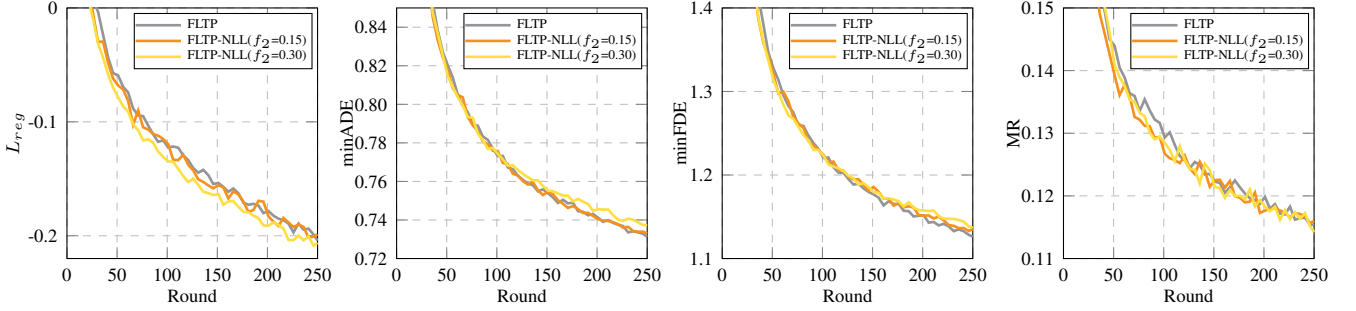

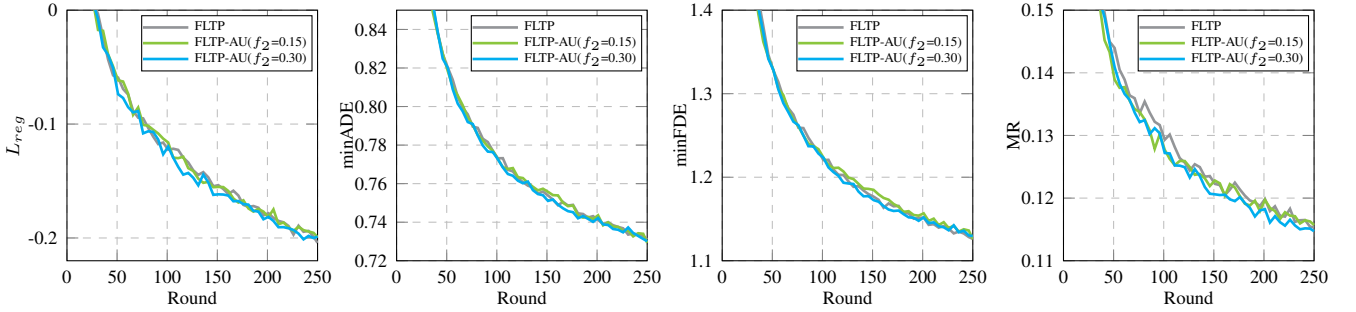
\begin{figure*}[ht]
\centering
\begin{subfigure}{0.27\linewidth}
\begin{tikzpicture}
\begin{axis}[
    xlabel={\scriptsize{Round}},
    ylabel={\scriptsize{$L_{reg}$}},
    every axis x label/.style={at={(current axis.south)},below=8pt},
    every axis y label/.style={
            at={(ticklabel* cs:0.5)},rotate=90,anchor=center,align=center, above=12.5pt},
    xmin=0, xmax=250,
    ymin=-0.22, ymax=0,
    xtick={0,50,100,150,200,250},
    xticklabels={\scriptsize{0},\scriptsize{50},\scriptsize{100},\scriptsize{150},\scriptsize{200},\scriptsize{250}},
    ytick={-0.2,-0.1,0,0.1,0.2,0.3,0.4,0.5,0.6},
    yticklabels={\scriptsize{-0.2},\scriptsize{-0.1},\scriptsize{0},\scriptsize{0.1},\scriptsize{0.2},\scriptsize{0.3},\scriptsize{0.3},\scriptsize{0.4},\scriptsize{0.5},\scriptsize{0.6}},
    height = 1\linewidth,
    width = 1\linewidth,
    xmajorgrids=true,
    ymajorgrids=true,
    grid style=dashed,
    legend cell align={left},
    legend style={inner sep=0pt,row sep=-3pt},
    legend style={font=\tiny},
    every axis plot/.append style={ultra thick}
]
    
    \addplot [
    color=Gray,
    line width=1pt
    ]
    table[x=roundshow,y=flloss,col sep = comma] {iros_data_noniid.csv};



    \addplot [
    color=LimeGreen,
    line width=1pt,
    ]
    table[x=roundshow,y=auloss0.15,col sep = comma] {iros_data_noniid.csv};

    \addplot [
    color=Cyan,
    line width=1pt,
    ]
    table[x=roundshow,y=auloss0.3,col sep = comma] {iros_data_noniid.csv};

\legend{FLTP,FLTP-AU($f_2$=0.15),FLTP-AU($f_2$=0.30)}

\end{axis}
\end{tikzpicture}
\vspace{-5pt}
\end{subfigure}
\hspace{-20pt}
\begin{subfigure}{0.27\linewidth}
\begin{tikzpicture}
\begin{axis}[
    xlabel={\scriptsize{Round}},
    ylabel={\scriptsize{minADE}},
    every axis x label/.style={at={(current axis.south)},below=8pt},
    every axis y label/.style={
            at={(ticklabel* cs:0.5)},rotate=90,anchor=center,align=center, above=13pt},
    xmin=0, xmax=250,
    ymin=0.72, ymax=0.85,
    xtick={0,50,100,150,200,250},
    xticklabels={\scriptsize{0},\scriptsize{50},\scriptsize{100},\scriptsize{150},\scriptsize{200},\scriptsize{250}},
    ytick={0.72,0.74,0.76,0.78,0.80,0.82,0.84},
    yticklabels={\scriptsize{0.72},\scriptsize{0.74},\scriptsize{0.76},\scriptsize{0.78},\scriptsize{0.80},\scriptsize{0.82},\scriptsize{0.84}},
    height = 1\linewidth,
    xmajorgrids=true,
    ymajorgrids=true,
    grid style=dashed,
    legend cell align={left},
    legend style={inner sep=0pt,row sep=-3pt},
    legend style={font=\tiny},
    every axis plot/.append style={ultra thick}
]
    
    \addplot [
    color=Gray,
    line width=1pt,
    ]
    table[x=roundshow,y=flade,col sep = comma] {iros_data_noniid.csv};



    \addplot [
    color=LimeGreen,
    line width=1pt,
    ]
    table[x=roundshow,y=auade0.15,col sep = comma] {iros_data_noniid.csv};

    \addplot [
    color=Cyan,
    line width=1pt,
    ]
    table[x=roundshow,y=auade0.3,col sep = comma] {iros_data_noniid.csv};

\legend{FLTP,FLTP-AU($f_2$=0.15),FLTP-AU($f_2$=0.30)}
    
\end{axis}
\end{tikzpicture}
\vspace{-5pt}
\end{subfigure}
\hspace{-20pt}
\begin{subfigure}{0.27\linewidth}
\begin{tikzpicture}
\begin{axis}[
    xlabel={\scriptsize{Round}},
    ylabel={\scriptsize{minFDE}},
    every axis x label/.style={at={(current axis.south)},below=8pt},
    every axis y label/.style={
            at={(ticklabel* cs:0.5)},rotate=90,anchor=center,align=center, above=10pt},
    xmin=0, xmax=250,
    ymin=1.1, ymax=1.4,
    xtick={0,50,100,150,200,250},
    xticklabels={\scriptsize{0},\scriptsize{50},\scriptsize{100},\scriptsize{150},\scriptsize{200},\scriptsize{250}},
    ytick={1.1,1.2,1.3,1.4},
    yticklabels={\scriptsize{1.1},\scriptsize{1.2},\scriptsize{1.3},\scriptsize{1.4}},
    height = 1\linewidth,
    xmajorgrids=true,
    ymajorgrids=true,
    grid style=dashed,
    legend cell align={left},
    legend style={inner sep=0pt,row sep=-3pt},
    legend style={font=\tiny},
    every axis plot/.append style={ultra thick}
]
    
    \addplot [
    color=Gray,
    line width=1pt,
    ]
    table[x=roundshow,y=flfde,col sep = comma] {iros_data_noniid.csv};



    \addplot [
    color=LimeGreen,
    line width=1pt,
    ]
    table[x=roundshow,y=aufde0.15,col sep = comma] {iros_data_noniid.csv};
    
    \addplot [
    color=Cyan,
    line width=1pt,
    ]
    table[x=roundshow,y=aufde0.3,col sep = comma] {iros_data_noniid.csv};

\legend{FLTP,FLTP-AU($f_2$=0.15),FLTP-AU($f_2$=0.30)}
    
\end{axis}
\end{tikzpicture}
\vspace{-5pt}
\end{subfigure}
\hspace{-20pt}
\begin{subfigure}{0.27\linewidth}
\begin{tikzpicture}
\begin{axis}[
    xlabel={\scriptsize{Round}},
    ylabel={\scriptsize{MR}},
    every axis x label/.style={at={(current axis.south)},below=8pt},
    every axis y label/.style={
            at={(ticklabel* cs:0.5)},rotate=90,anchor=center,align=center, above=13pt},
    xmin=0, xmax=250,
    ymin=0.11, ymax=0.15,
    xtick={0,50,100,150,200,250},
    xticklabels={\scriptsize{0},\scriptsize{50},\scriptsize{100},\scriptsize{150},\scriptsize{200},\scriptsize{250}},
    ytick={0.11,0.12,0.13,0.14,0.15},
    yticklabels={\scriptsize{0.11},\scriptsize{0.12},\scriptsize{0.13},\scriptsize{0.14},\scriptsize{0.15}},
    height = 1\linewidth,
    xmajorgrids=true,
    ymajorgrids=true,
    grid style=dashed,
    legend cell align={left},
    legend style={inner sep=0pt,row sep=-3pt},
    legend style={font=\tiny},
    every axis plot/.append style={ultra thick}
]
    
    \addplot [
    color=Gray,
    line width=1pt,
    ]
    table[x=roundshow,y=flmr,col sep = comma] {iros_data_noniid.csv};



    \addplot [
    color=LimeGreen,
    line width=1pt,
    ]
    table[x=roundshow,y=aumr0.15,col sep = comma] {iros_data_noniid.csv};

    \addplot [
    color=Cyan,
    line width=1pt,
    ]
    table[x=roundshow,y=aumr0.3,col sep = comma] {iros_data_noniid.csv};

\legend{FLTP,FLTP-AU($f_2$=0.15),FLTP-AU($f_2$=0.30)}
    
\end{axis}
\end{tikzpicture}
\vspace{-5pt}
\end{subfigure}
\caption{Round-wise comparison between FLTP and FLTP-AU}
\label{fig:alfltpau}
\vspace{-5pt}
\end{figure*}

\begin{table}[ht]
\setlength\tabcolsep{2pt}
  \centering
  \renewcommand{\arraystretch}{1.1}
  \begin{tabular}{l|c|c|c|c|c}
    \hline
    Model & Round & NLL($\downarrow$) & minADE($\downarrow$) & minFDE($\downarrow$) & MR($\downarrow$)\\
    \hline
    Centralized HiVT \cite{zhou2022hivt} & - & 0.467 & 0.685 & 1.028 & 0.104\\
    \hline
    Client 0 w/o FL & 50 & 0.897 & 0.992 & 1.732 & 0.207\\
    FLTP & 50 & 0.632 & \textbf{0.818} & \textbf{1.321} & 0.143\\ 
    FLTP-NLL($f_2$=0.15) & 50 & 0.634 & 0.821 & 1.327 & \textbf{0.141}\\
    FLTP-NLL($f_2$=0.30) & 50 & 0.637 & 0.819 & \textbf{1.321} & 0.143\\
    FLTP-AU($f_2$=0.15) & 50 & 0.632 & 0.819 & 1.325 & \textbf{0.141}\\
    FLTP-AU($f_2$=0.30) & 50 & \textbf{0.631} & \textbf{0.818} & 1.329 & 0.143\\
    \hline
    Client 0 w/o FL & 150 & 1.098 & 1.026 & 1.822 & 0.229\\
    FLTP & 150 & 0.556 & 0.754 & 1.181 & 0.124\\ 
    FLTP-NLL($f_2$=0.15) & 150 & \textbf{0.555} & \textbf{0.753} & 1.182 & 0.122\\
    FLTP-NLL($f_2$=0.30) & 150 & 0.564 & 0.754 & 1.177 & 0.122\\
    FLTP-AU($f_2$=0.15) & 150 & 0.556 & \textbf{0.753} & 1.177 & \textbf{0.121}\\
    FLTP-AU($f_2$=0.30) & 150 & 0.556 & \textbf{0.753} & \textbf{1.176} & \textbf{0.121}\\
    \hline
    Client 0 w/o FL & 250 & 1.259 & 1.059 & 1.896 & 0.245\\
    FLTP & 250 & \textbf{0.528} & 0.731 & \textbf{1.126} & 0.115\\ 
    FLTP-NLL($f_2$=0.15) & 250 & 0.533 & 0.733 & 1.135 & 0.116\\
    FLTP-NLL($f_2$=0.30) & 250 & 0.541 & 0.737 & 1.137 & \textbf{0.114}\\
    FLTP-AU($f_2$=0.15) & 250 & \textbf{0.528} & \textbf{0.729} & \textbf{1.126} & 0.116\\
    FLTP-AU($f_2$=0.30) & 250 & 0.529 & 0.730 & 1.130 & 0.115\\
    \hline
  \end{tabular}
  \caption{Performance on Argoverse Validation Set. Fraction of clients selected for communication in each round $f_1$ is set to be 0.1.}
  \label{i^*b:roundshotdata}
  \vspace{-5pt}
\end{table}

\subsection{Comparison of Uncertainty-aware Selectors}
From Fig. \ref{fig:alfltpnll} and Fig. \ref{fig:alfltpau},  we observe the following:
\begin{itemize}
    \item \textbf{Convergence speed of training loss:} The regression losses of both FLTP-NLL and FLTP-AU with $f_2=0.30$ converge faster than with $f_2=0.15$ and vanilla FLTP.
    \item \textbf{Round-wise validation performance:} FLTP-AU with $f_2=0.30$ performs better than FLTP in most rounds in terms of MR. As MR measures the fraction of scenarios with endpoint errors larger than 2 meters, a lower MR indicates that FLTP-AU is more robust to various traffic scenarios in the inference stage.
    \item \textbf{Impact of biased selection-NLL:} After around 100 rounds, FLTP surpasses the performance of FLTP-NLL with both $f_2 = 0.15$ and $f_2 = 0.30$ in terms of minADE and minFDE due to the biased selection. Notably, FLTP-NLL with $f_2 = 0.30$ performs worse than with $f_2 = 0.15$ because the former introduces more bias.
    \item \textbf{Impact of biased selection-AU:} 
    FLTP-AU with $f_2=0.30$ achieves comparable minADE and minFDE to FLTP while outperforming in terms of MR in most rounds. This indicates that a larger $f_2$ helps FLTP-AU to select clients with more representative data, improving robustness.
\end{itemize}

 Table \ref{i^*b:roundshotdata} shows the detailed global model performance of specific rounds, where we can see that: 
 \begin{itemize}
    \item In the 50th round, FLTP-AU with $f_2=0.15$ outperforms other frameworks in terms of MR, while FLTP-AU with $f_2=0.30$ outperforms other frameworks in terms of NLL and minADE.
    \item In the 150th round, FLTP-AU with both $f_2=0.15$ and $f_2=0.30$ outperform other frameworks in terms of MR.
    \item In the 250th round, where global models finish training, FLTP-AU with $f_2=0.15$ outperforms other frameworks in terms of NLL, minADE and minFDE.
    \item Although FLTP-based HiVT models in the 250th round perform slightly worse than the centralized HiVT, they ensure the privacy of human-driven vehicles by avoiding data exchange with the server.
 \end{itemize}

In summary, FLTP-AU converges faster in regression loss and achieves better performance in terms of MR than FLTP in most rounds. Moreover, FLTP-AU shows better and more stable round-wise performance than FLTP-NLL.

\subsection{Overall Performance Comparison under HPNet}

As mentioned in Section~\ref{subsec: active learning} and Section~\ref{subsec: complexity selector},
We adopt a complexity-scoring mechanism to rank each trajectory by its complexity. To simulate realistic non-IID data distributions across autonomous vehicles, we partition the Argoverse dataset into \( C \) clients, where each client \( c \in \mathcal{C} \triangleq \{1, \dots, C\} \) represents one ego vehicle collecting local driving scenario data. We use a Dirichlet distribution parameterized by \( \alpha \) to assign trajectory samples to clients. The parameter \( \alpha \) controls the degree of data heterogeneity: a smaller \( \alpha \) results in highly imbalanced data partitions across clients, while a larger \( \alpha \) leads to more uniform distributions.

To evaluate the robustness of different client selection strategies under varying levels of heterogeneity, we consider three different Dirichlet settings, namely \( \alpha = 0.1 \), \( \alpha = 0.5 \), and \( \alpha = 1 \). The setting \( \alpha = 0.1 \) represents a highly heterogeneous and severely non-IID environment, where clients contain highly skewed trajectory complexity distributions. In contrast, \( \alpha = 1 \) corresponds to a relatively balanced distribution across clients, while \( \alpha = 0.5 \) represents a moderate level of heterogeneity. This experimental design enables a comprehensive evaluation of how effectively different client selection methods adapt to varying degrees of client diversity and trajectory complexity imbalance.

\begin{table}[ht]
\centering
\caption{Performance comparison across selection methods (selection ratio $f_1=30\%$).}
\label{tab:performance}
\begin{tabular}{lcccc}
\toprule
\hline
\textbf{Method}  & \textbf{\( \alpha \)} & \textbf{minADE} & \textbf{minFDE} & \textbf{Miss Rate (MR)} \\
\midrule

Random  & 0.1 & 0.754 & 1.310 & 0.121 \\
High Loss  & 0.1& 0.703 & 1.080 & 0.115 \\
FLTP-NLL & 0.1& 0.712 & 1.140 & 0.101 \\
High Rank + Random & 0.1 & 0.691 & 1.010 & 0.103 \\
\textbf{{\our}} & 0.1 & \textbf{0.669} & \textbf{0.910} &  \textbf{0.101} \\  \hline
Random  & 0.5 & 0.722 & 1.210 & 0.121 \\
High Loss  & 0.5& 0.703 & 1.080 & 0.112 \\
FLTP-NLL & 0.5& 0.712 & 1.140 & 0.118 \\
High Rank + Random & 0.5 & 0.701 & 1.010 & 0.102 \\
\textbf{{\our}} & 0.5 & \textbf{0.682} & \textbf{0.891} & \textbf{0.096} \\ \hline
Random  & 1 & 0.701 & 1.121 & 0.103 \\
High Loss  & 1& 0.703 & 1.080 & 0.115 \\
FLTP-NLL & 1& 0.704 & 1.140 & 0.101 \\
High Rank + Random & 1 & 0.703 & 1.007 & 0.097 \\
\textbf{{\our}} & 1 & \textbf{0.698} & \textbf{0.901} &  \textbf{0.093} \\ \hline
\bottomrule
\end{tabular}
\end{table}

Table~\ref{tab:performance} summarizes the performance comparison of different client selection methods under the HPNet model, specifically evaluated using minADE, minFDE, and MR with a fixed selection ratio of 30\% (i.e., $f_1 = 0.3$).
Our proposed method, {\our}, consistently outperforms the baseline methods across all evaluation metrics under the three Dirichlet settings. 
Under the highly non-IID setting (\(\alpha = 0.1\)), {\our} achieves the best overall performance, obtaining the lowest minADE (0.669), minFDE (0.910), and MR (0.101). Compared with Random selection, it reduces minADE by 11.3\% and minFDE by 30.5\%, demonstrating the benefit of complexity-aware client selection under severe data heterogeneity.
As \(\alpha\) increases from 0.5 to 1, the performance gap among methods narrows because client data distributions become more balanced. Nevertheless, {\our} consistently achieves the best results across all heterogeneity levels, attaining the lowest minADE, minFDE, and MR for both \(\alpha = 0.5\) and \(\alpha = 1\). This indicates that the proposed dynamic ranking strategy effectively balances trajectory complexity and client diversity, resulting in robust and stable performance under varying degrees of non-IID data distributions.

Figures~\ref{fig:alpha01}-\ref{fig:alpha1} further illustrate the training progression across communication rounds. We observe that Random selection yields the worst performance across all three prediction metrics. In contrast, {\our} consistently achieves superior performance throughout all training rounds, demonstrating improved convergence stability and robustness compared with the baseline methods. Additionally, {\our} not only reduces minADE, but also further decreases minFDE compared to the other selection strategies. This indicates that {\our} effectively reduces prediction errors, especially in the worst-case trajectory prediction scenarios.

Moreover, the convergence curves show that {\our} maintains more stable optimization behavior across different heterogeneity settings. 
Specifically, under smaller \( \alpha \) values, where client distributions are more imbalanced, the baseline methods exhibit larger fluctuations and slower convergence, whereas {\our} converges more steadily and achieves lower final prediction errors. These observations further demonstrate the effectiveness of our proposed dynamic ranking strategy in mitigating the negative impact of client heterogeneity during federated trajectory prediction training.

\begin{figure*}[!t]
\centering

\subfloat[minADE\label{fig:minADE01}]{
    \includegraphics[width=0.31\textwidth]{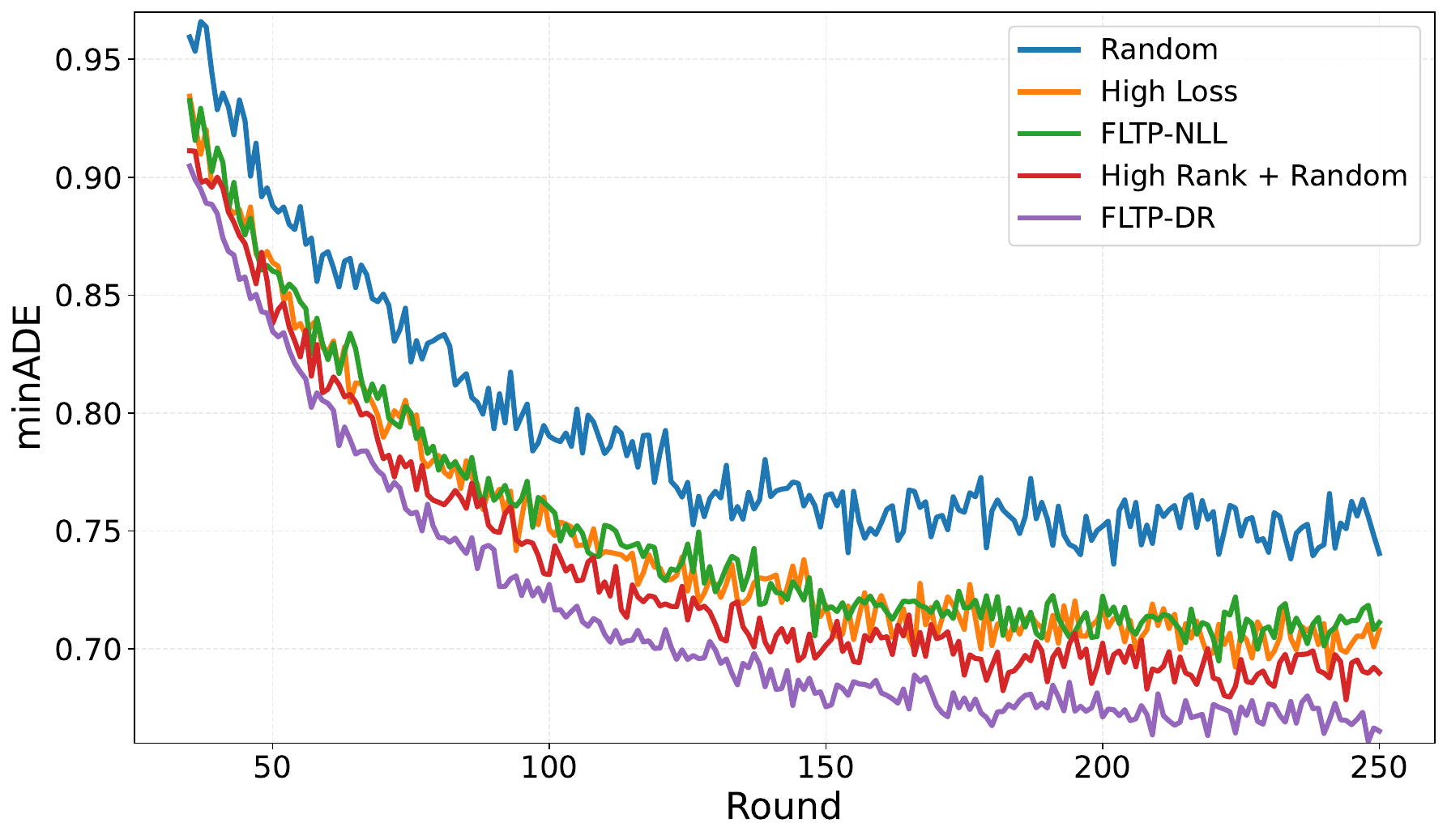}
}
\hfill
\subfloat[minFDE\label{fig:minFDE01}]{
    \includegraphics[width=0.31\textwidth]{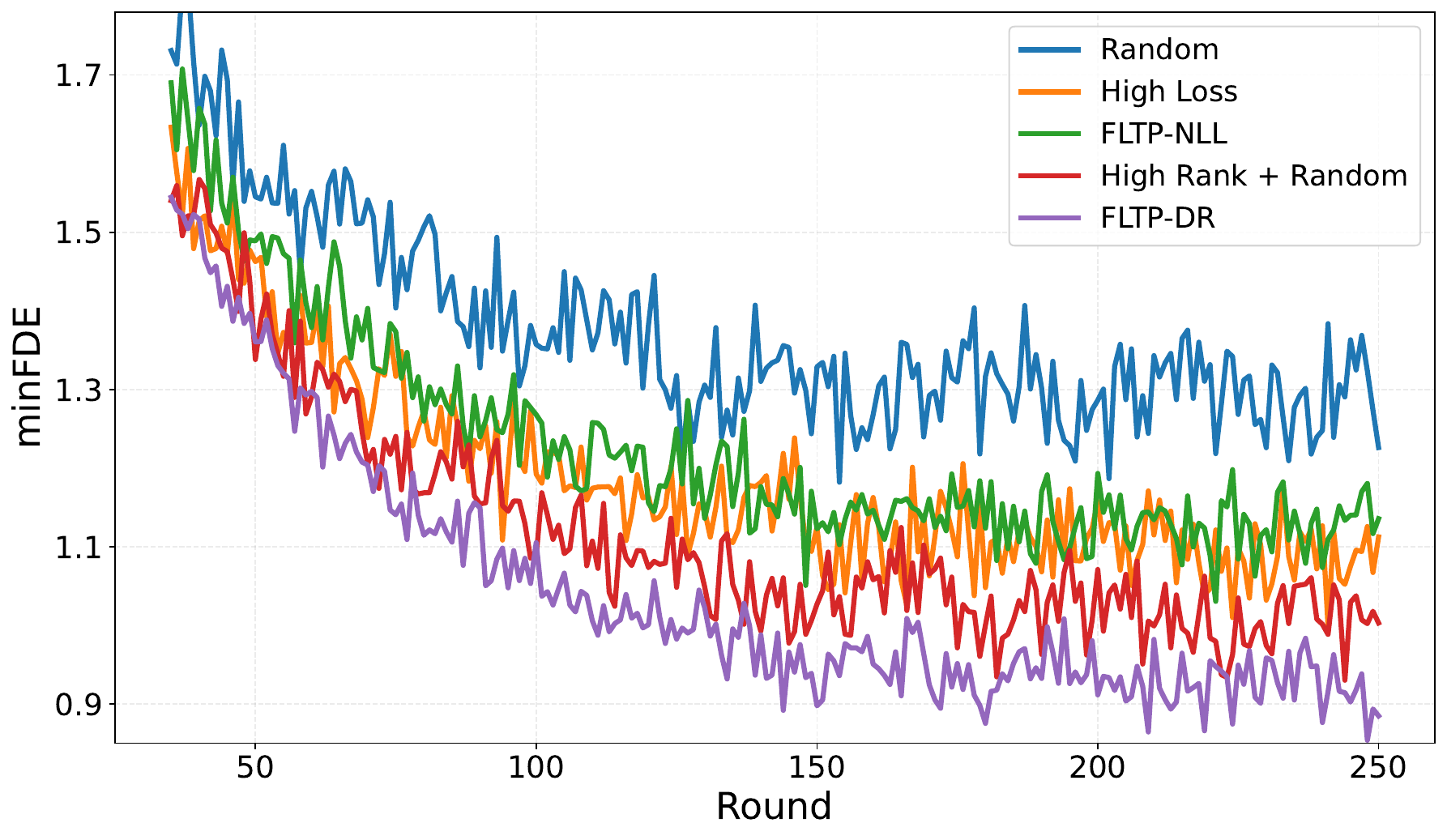}
}
\hfill
\subfloat[MR\label{fig:MR01}]{
    \includegraphics[width=0.31\textwidth]{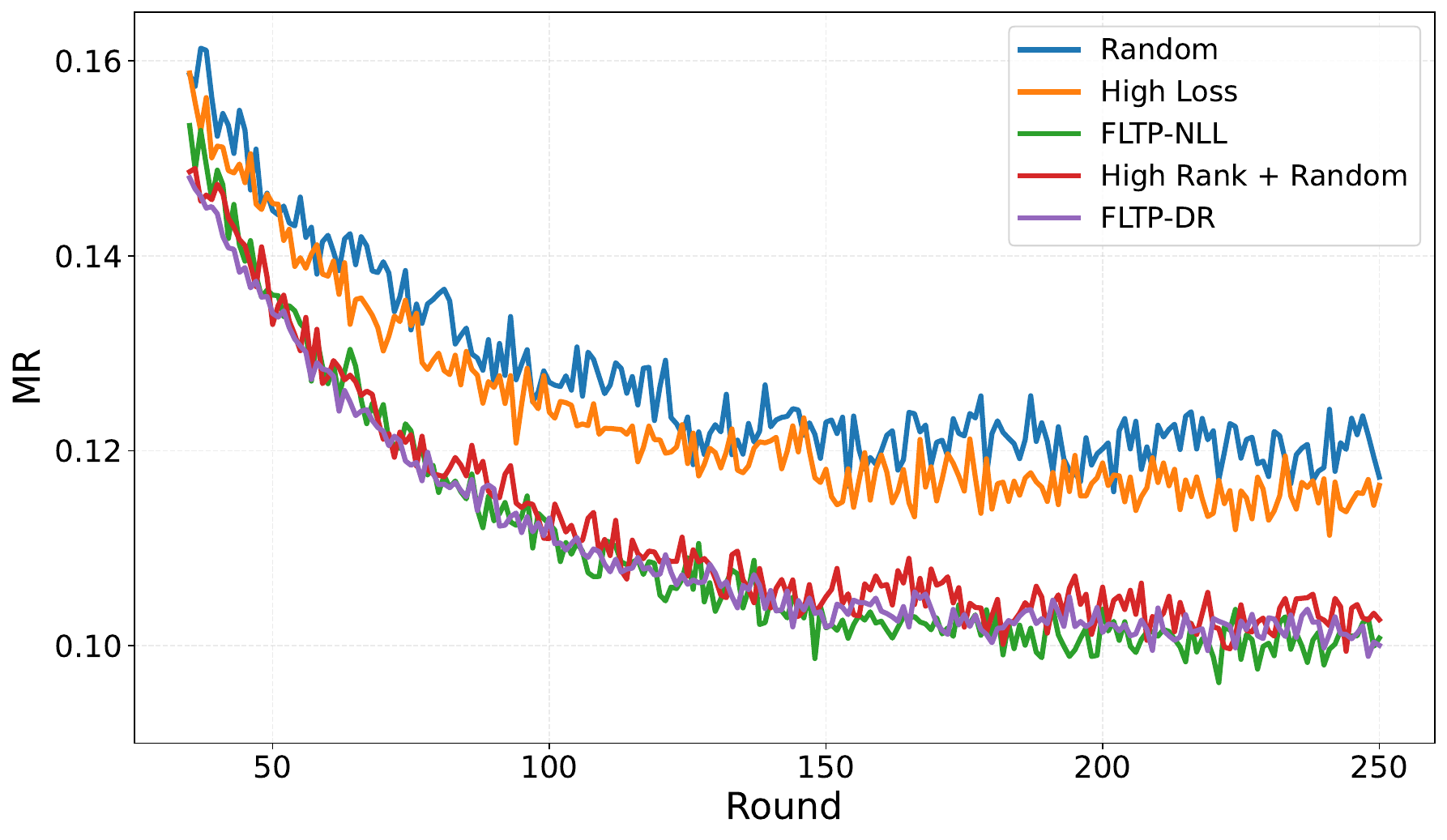}
}

\caption{Convergence behavior under $\alpha=0.1$.}
\label{fig:alpha01}
\end{figure*}

\begin{figure*}[!t]
\centering

\subfloat[minADE\label{fig:minADE05}]{
    \includegraphics[width=0.31\textwidth]{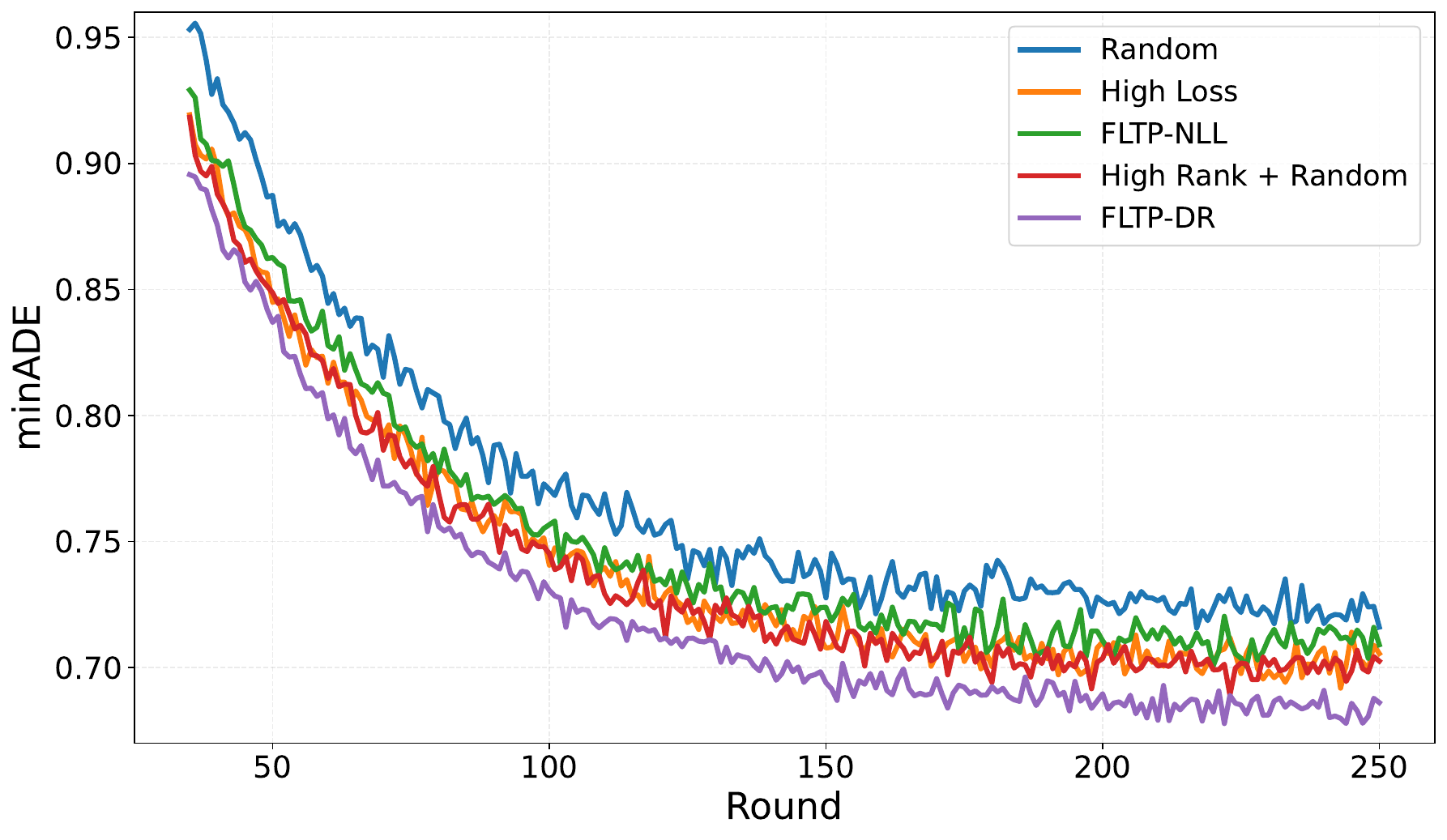}
}
\hfill
\subfloat[minFDE\label{fig:minFDE05}]{
    \includegraphics[width=0.31\textwidth]{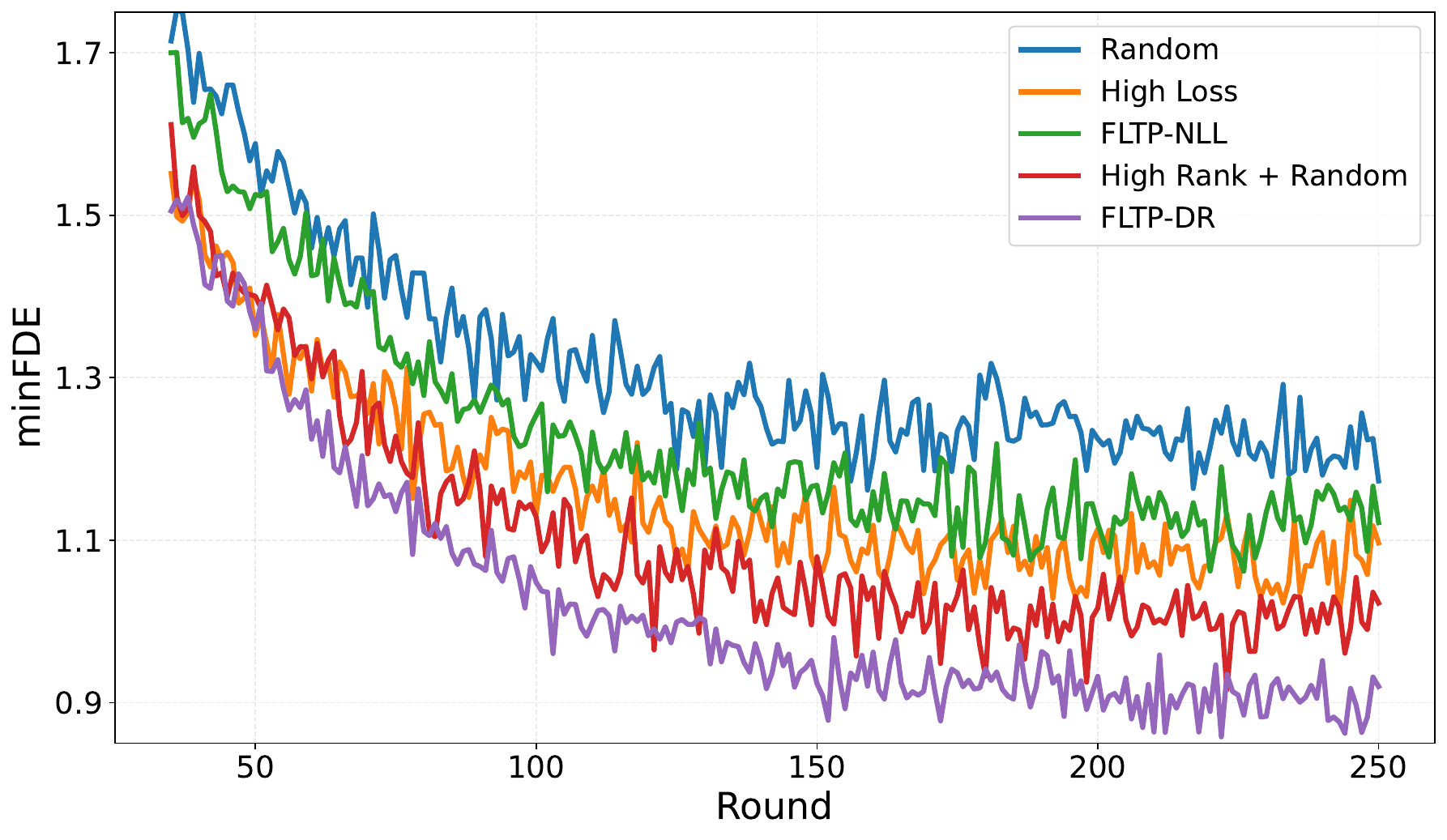}
}
\hfill
\subfloat[MR\label{fig:MR05}]{
    \includegraphics[width=0.31\textwidth]{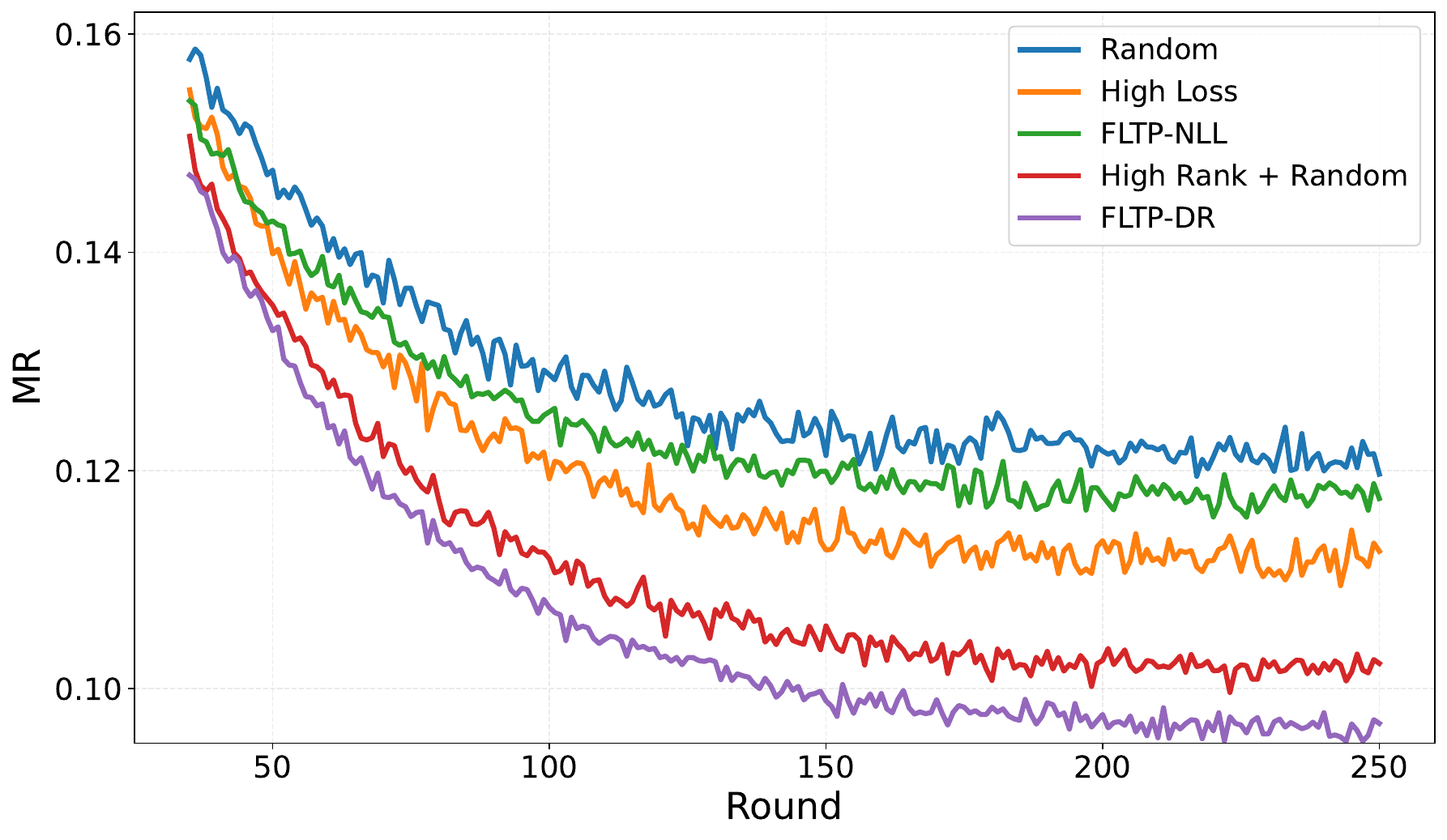}
}

\caption{Convergence behavior under $\alpha=0.5$.}
\label{fig:alpha05}
\end{figure*}

\begin{figure*}[!t]
\centering

\subfloat[minADE\label{fig:minADE1}]{
    \includegraphics[width=0.31\textwidth]{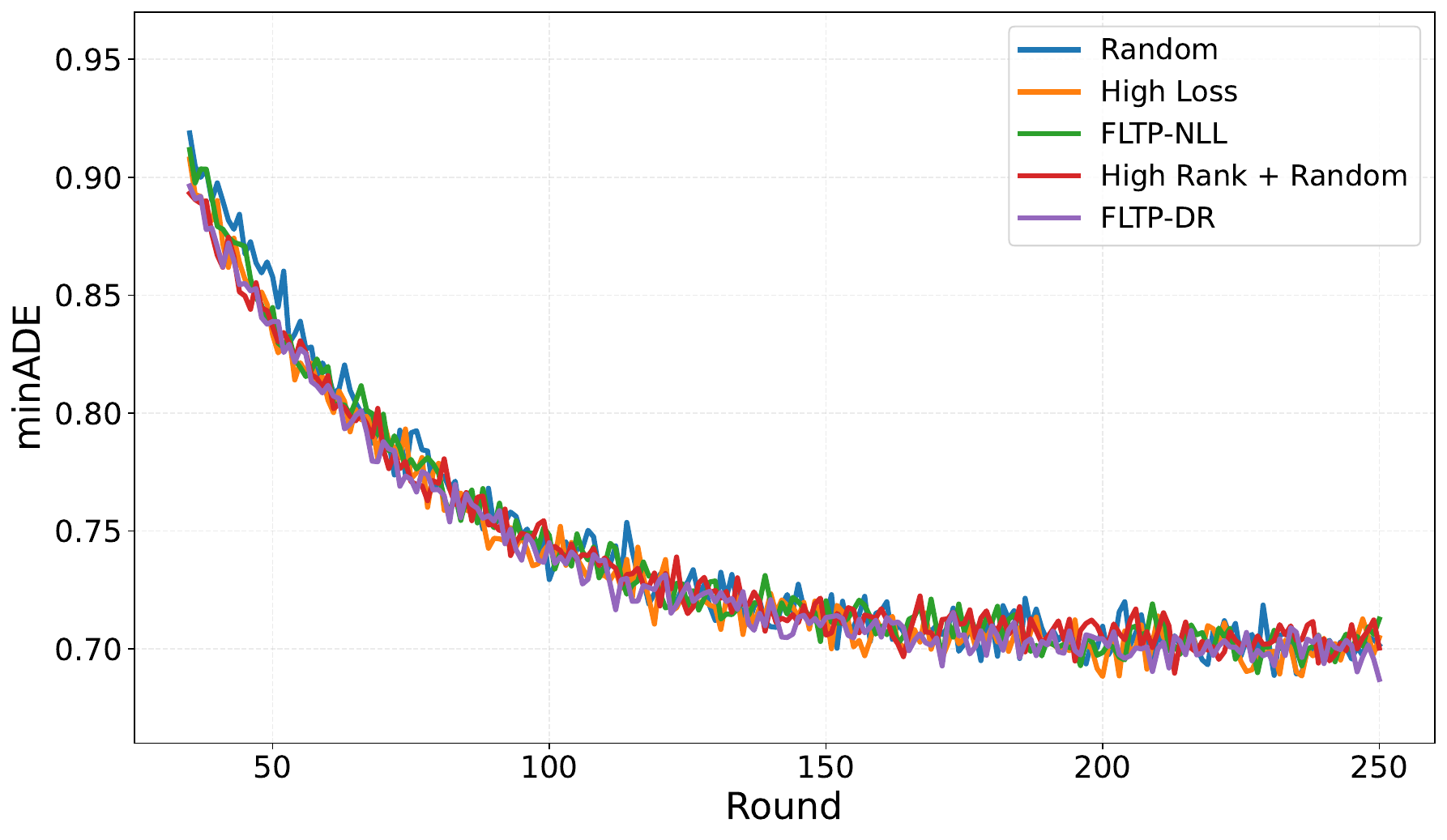}
}
\hfill
\subfloat[minFDE\label{fig:minFDE1}]{
    \includegraphics[width=0.31\textwidth]{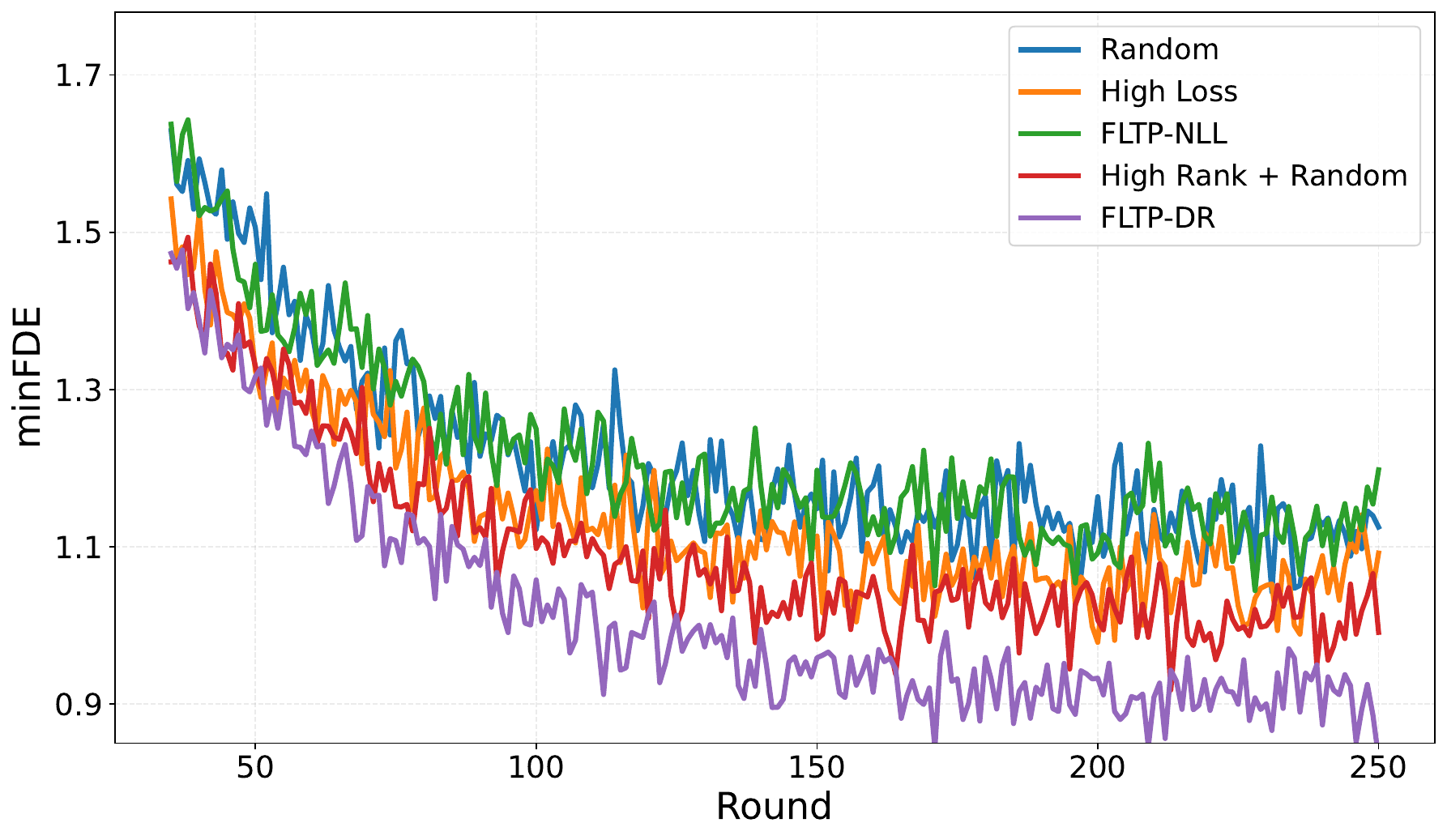}
}
\hfill
\subfloat[MR\label{fig:MR1}]{
    \includegraphics[width=0.31\textwidth]{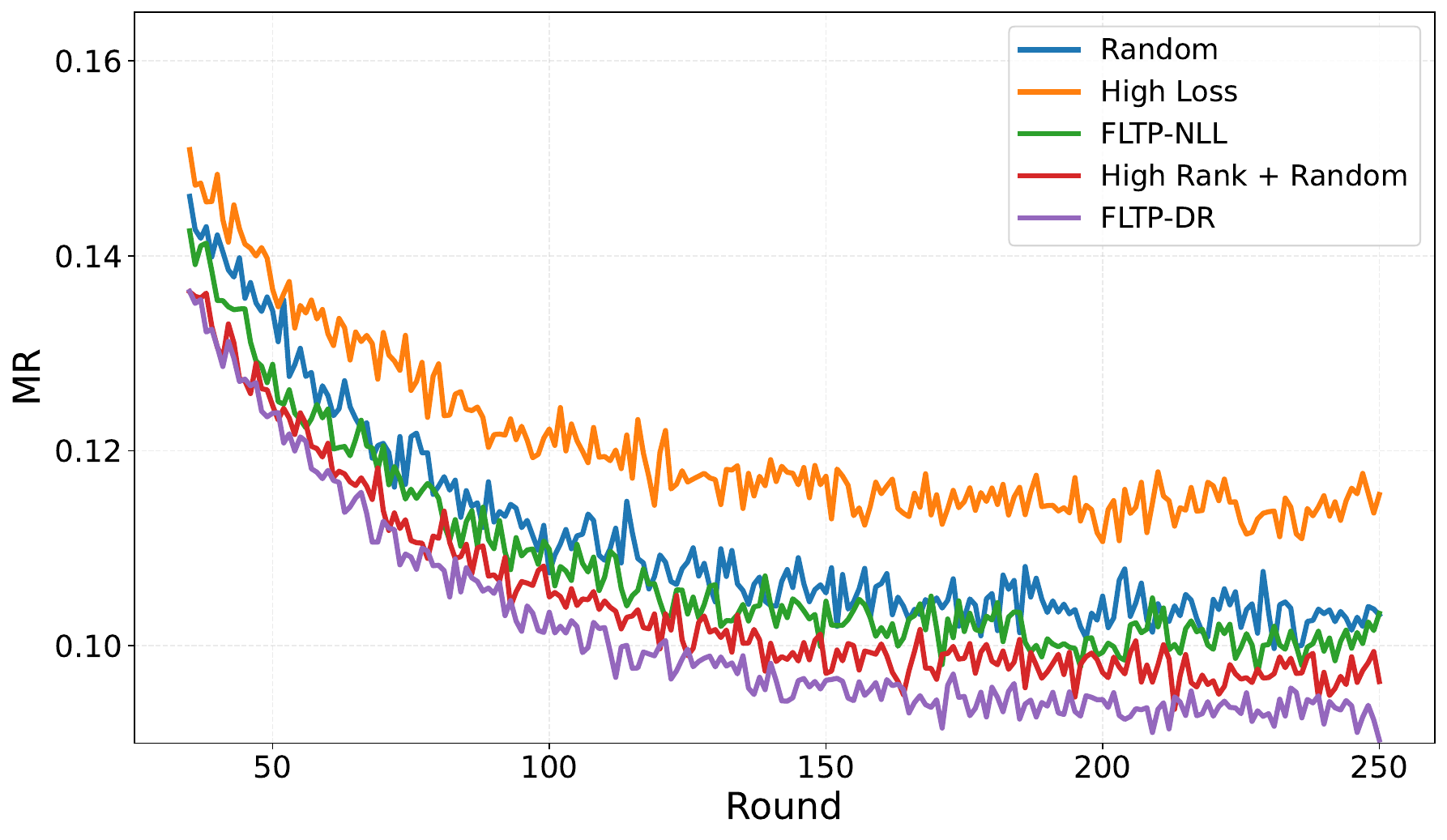}
}

\caption{Convergence behavior under $\alpha=1$.}
\label{fig:alpha1}
\end{figure*}

\subsection{Impact of Training and Validation Dataset  Complexity}

We examine how \emph{training-set  complexity} shapes robustness to trajectory dynamics under a \emph{fixed, mixed-complexity validation set} \(\mathcal{D}_{\mathrm{val}}^{\mathrm{mix}}\). 
Let
\[
\mathsf{Mix}(p_H,p_M,p_E)
\quad\text{with}\quad
p_H+p_M+p_E=1
\]
denote a training mixture drawn from the Hard/Moderate/Easy strata in proportions \((p_H,p_M,p_E)\). 
We vary the hard fraction \(n:=p_H\in\{0.25,0.50,0.75\}\), and place the remainder \((1-n)\) either on E (\(p_E=1-n\)) or on M (\(p_M=1-n\)). 
In our setup, higher  complexity correlates with richer scene or behavioral coverage—more (including sharp) turns, shorter distances to intersections, multi-lane merges, stop-and-go segments with larger speed variance, and denser agent interactions—whereas easier sequences are dominated by quasi-linear motion with fewer topology changes. 
Training on harder trajectories therefore exposes the forecaster to a broader set of kinematic regimes and map contexts, alleviating straight-line bias and improving transfer to simpler regimes that are effectively subsets of the complex ones.

As shown in Table~\ref{tab:consolidated}, Panel~A, increasing $p_H$ from $0.25$ to $0.75$ produces a consistent trend on the same validation distribution: both minADE and minFDE decrease, while MR also shows a downward tendency.  
For instance, shifting from $\mathsf{Mix}(0.25,0,0.75)$ to $\mathsf{Mix}(0.75,0.25,0)$ improves minADE from $0.738$ to $0.697$ (approximately $5.56\%$ relative), minFDE from $1.210$ to $1.096$ (approximately $9.42\%$), and MR from $0.119$ to $0.114$ (approximately $4.20\%$).  
When $n$ is fixed, assigning the remaining proportion to M is at least as effective, and often marginally more beneficial, than assigning it to E. For example, at $n{=}0.5$, $\mathsf{Mix}(0.50,0.50,0)$ compared with $\mathsf{Mix}(0.50,0,0.50)$ improves minADE by approximately $1.53\%$, minFDE by approximately $1.17\%$, and MR by approximately $3.39\%$.   

Overall, {\our} benefits from a larger share of challenging trajectories during training: the added maneuver and topology coverage improves generalization not only to the mixed validation set but also to its moderate or easy subsets.

Beyond varying \(p_H\) under a fixed validation set, we further disentangle the roles of training and validation  complexity. Figure~\ref{fig:complexity_analysis} summarizes two complementary views: 
(i) in Figure~\ref{fig:complexity_analysis}(a), we vary the \emph{training}  complexity (E/M/H) while validating on the same \(\mathcal{D}_{\mathrm{val}}^{\mathrm{mix}}\); 
(ii) in Figure~\ref{fig:complexity_analysis}(b), we fix training to a mixed-complexity set and vary the \emph{validation}  complexity. 
Both views corroborate a knowledge-transfer effect: exposure to harder scenes during training broadens maneuver coverage and yields improved accuracy and robustness, including in moderate and easy validation regimes.

\begin{table}[!t]
\footnotesize
\renewcommand{\arraystretch}{1.15}
\setlength{\tabcolsep}{5.5pt}
\caption{Consolidated results. Panel A: impact of training mixture $\mathsf{Mix}(p_H,p_M,p_E)$ on a fixed mixed- complexity validation set $\mathcal{D}_{\mathrm{val}}^{\mathrm{mix}}$. Panel B: effect of selection ratio $f_1$. Panel C: effect of balance parameter $n$ in {\our}.}
\label{tab:consolidated}
\centering
\begin{tabular*}{\columnwidth}{@{\extracolsep{\fill}} llccc}
\toprule
\textbf{Panel} & \textbf{Setting} & \textbf{minADE} & \textbf{minFDE} & \textbf{MR} \\\hline
\midrule
\multicolumn{5}{l}{\emph{Panel A: Training mixture } $\mathsf{Mix}(p_H,p_M,p_E)$} \\
& $\mathsf{Mix}(0.75,0,0.25)$  & 0.698 & 1.098 & 0.115 \\
& {\bf $\mathsf{Mix}(0.75,0.25,0)$}  & \textbf{0.697} & \textbf{1.096} & \textbf{0.114} \\
& $\mathsf{Mix}(0.50,0,0.50)$  & 0.721 & 1.110 & 0.118 \\
& $\mathsf{Mix}(0.50,0.50,0)$  & 0.710 & 1.097 & 0.114 \\
& $\mathsf{Mix}(0.25,0,0.75)$  & 0.738 & 1.210 & 0.119 \\
& $\mathsf{Mix}(0.25,0.75,0)$  & 0.721 & 1.118 & 0.115 \\\hline
\addlinespace 
\multicolumn{5}{l}{\emph{Panel B: Selection ratio } $f_1$} \\
& 10\%                           & 0.728 & 1.230 & 0.123 \\
& 20\%                           & 0.715 & 1.120 & 0.103 \\
& \textbf{30\% }       & \textbf{0.682} & \textbf{0.891} & \textbf{0.096} \\\hline
\addlinespace
\multicolumn{5}{l}{\emph{Panel C: Balance parameter } $\nu$} \\
& 0                              & 0.693 & 1.110 & 0.114 \\
& 10                             & 0.690 & 0.898 & 0.099 \\
& \textbf{15}                    & \textbf{0.682} & \textbf{0.891} & \textbf{0.096} \\
& 30                             & 0.703 & 1.080 & 0.112 \\
\bottomrule
\end{tabular*}
\end{table}

\begin{figure}[!t]
    \centering
    \includegraphics[width=\columnwidth]{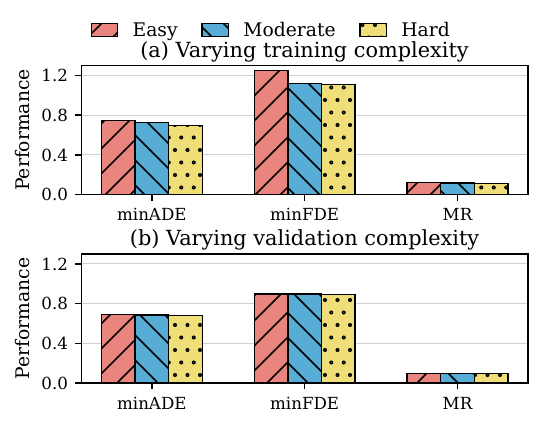}
    \caption{Impact of training and validation complexity on forecasting metrics. Training with \textbf{hard} workloads reduces minADE/minFDE, indicating that exposure to complex scenarios yields more robust representations. Evaluating on \textbf{hard} validation data is most demanding (and thus most discriminative), while performance remains strong on \textbf{easy}/\textbf{moderate} sets, evidencing effective generalization.}
    \label{fig:complexity_analysis}
\end{figure}

\subsection{Sensitivity Analysis of {\our}}

We assess how {\our} behaves under two knobs that matter in practice: the \emph{per-round participation budget} and the \emph{mixing rule} between complexity and loss.

\subsubsection{Impact of Selection Ratio ($f_1$)}
Let $C$ be the number of available clients and $\mathcal{L}_r=\lfloor f_1\,C\rfloor$ the number of selected clients per round. 
A larger $f_1$ increases \emph{scene diversity} per aggregation (more maps, traffic regimes, and behaviors), reduces the variance of the aggregated update, and stabilizes uncertainty estimates (AU/NLL) by averaging over more clients; however, it also incurs higher communication costs.
Table~\ref{tab:consolidated}, Panel~B, shows a clear trend: increasing participation from $10\%$ to $30\%$ monotonically lowers minADE/minFDE and MR. 
With $10\%$, the small cohort over-represents a few clients and impedes coverage of rare/hard scenes; $30\%$ yields broader coverage and better generalization, providing a good accuracy–cost trade-off in our setting.

\subsubsection{Impact of the Mixing Parameter ($\nu$) in Client Selection}
At round $r$, given candidate pool $\mathcal{C}_t$ and selection budget $\mathcal{L}_r$, the deterministic hybrid selector (Sec.~\ref{sec: data partition}) forms $\mathcal{L}_r$ participants \emph{without replacement} by combining complexity coverage and loss exploitation.

Table~\ref{tab:consolidated}, Panel~C, shows that extremes are suboptimal: 
$\nu{=}0$ (pure loss) tends to over-exploit noisy clients and under-cover systematically hard but informative regions; 
$\nu{=}|\mathcal{L}_r|$ (pure complexity) may repeatedly pick already well-fit hard clients, yielding weak gradients and slower error reduction. 
A \emph{balanced} mix (e.g., $\nu{=}15$ for $|\mathcal{L}_r|{=}30$) prioritizes clients that are both challenging and currently underfit, producing the best trade-off in minADE, minFDE, and MR.

\section{Conclusion}

In this paper, we first proposed a privacy-preserving and uncertainty-aware federated learning framework (FLTP) tailored for trajectory prediction tasks in connected autonomous vehicles, initially leveraging a global objective with uncertainty quantification. Building upon this, we further employed two uncertainty-aware metrics—negative log-likelihood (NLL) and aleatoric uncertainty (AU)—for adaptive client selection to optimize partial client participation. Extending these prior advancements, we presented {\our}, an enhanced adaptive federated learning method integrating hierarchical prediction networks (HPNet) to evaluate trajectory  complexity during data partitioning and client selection. {\our} effectively balances the contribution of high-rank and high-loss clients, significantly improving model generalization, robustness, and training efficiency.

Extensive experiments conducted on the Argoverse dataset consistently demonstrated the superior performance of {\our} compared to baseline methods across multiple metrics, including minADE, minFDE, and Miss Rate. Additionally, our sensitivity analysis confirmed that {\our} provides stable and improved performance under varying client selection ratios and different levels of trajectory  complexity. Future research will explore incorporating additional advanced prediction models and refining client selection strategies to further enhance federated learning's performance and efficiency in real-world autonomous driving applications.



\bibliographystyle{IEEEtran}
\bibliography{cumtom} 


\begin{IEEEbiographynophoto}{Yiming Xie}
is a Ph.D. student in the Department of Electrical and Computer Engineering at Northeastern University, Boston, Massachusetts. His research focuses on federated learning optimization schemes, data loader allocation for distributed deep learning, and NVMe over Fabrics SSDs.
\end{IEEEbiographynophoto}

\begin{IEEEbiographynophoto}{Muzi Peng}
was a Ph.D. student at Northeastern University, majoring in Electrical and Computer Engineering. He was primarily working on autonomous vehicle systems, with a particular focus on surrounding perception. He graduated with a B.S. in 2022 from Nanjing University.
\end{IEEEbiographynophoto}

\begin{IEEEbiographynophoto}{Fei Miao}
is the Pratt \& Whitney Associate Professor of the School of Computing and a Courtesy Faculty of the Department of Electrical and Computer Engineering at the University of Connecticut. 
She received her Ph.D. degree and the Best Doctoral Dissertation Award in Electrical and Systems Engineering, with a dual M.S. degree in Statistics from the University of Pennsylvania in 2016. She was a postdoctoral researcher at the GRASP Lab and the PRECISE Lab of UPenn from 2016 to 2017. 
\end{IEEEbiographynophoto}

\begin{IEEEbiographynophoto}{Ningfang Mi}
is an Associate Professor in the Department of Electrical and Computer Engineering at Northeastern University. She received her Ph.D. in Computer Science from the College of William and Mary in 2009 under the guidance of Prof. Evgenia Smirni, her M.S. in Computer Science from the University of Texas at Dallas in 2004, and her B.S. in Computer Science from Nanjing University in 2000. 
\end{IEEEbiographynophoto}

\begin{IEEEbiographynophoto}{Lili Su}
is an Assistant Professor in the Department of Electrical and Computer Engineering at Northeastern University. 
She received her M.Sc. (2014) and Ph.D. (2017) in Electrical and Computer Engineering from the University of Illinois at Urbana–Champaign (UIUC). Prior to joining Northeastern, she was a postdoctoral researcher at the Computer Science and Artificial Intelligence Laboratory (CSAIL) at MIT. 
\end{IEEEbiographynophoto}

\end{document}